\documentclass{article}

\newif\ifpreprint\preprinttrue

\PassOptionsToPackage{noend}{algorithmic}
\PassOptionsToPackage{dvipsnames}{xcolor}

\usepackage[utf8]{inputenc}
\usepackage[T1]{fontenc}
\usepackage{microtype}
\usepackage{graphicx}
\usepackage{subcaption}
\usepackage{wrapfig}
\usepackage{booktabs}
\usepackage{multirow}
\usepackage{makecell}
\usepackage{enumitem}
\usepackage{threeparttable}
\usepackage{url}
\usepackage{nicefrac}
\usepackage{iclr2027_conference,times}
\usepackage{hyperref}

\usepackage{amsmath}
\usepackage{amssymb}
\usepackage{mathtools}
\usepackage{amsthm}

\usepackage{amsmath,amsfonts,bm}

\def\eqref#1{equation~\ref{#1}}

\def\1{\bm{1}}

\def\ve{{\bm{e}}}

\def\vx{{\bm{x}}}

\def\vz{{\bm{z}}}

\def\mP{{\bm{P}}}

\DeclareMathAlphabet{\mathsfit}{\encodingdefault}{\sfdefault}{m}{sl}
\SetMathAlphabet{\mathsfit}{bold}{\encodingdefault}{\sfdefault}{bx}{n}

\newcommand{\E}{\mathbb{E}}

\newcommand{\KL}{D_{\mathrm{KL}}}

\usepackage[capitalize,noabbrev]{cleveref}
\crefname{appendix}{Appendix}{Appendices}\crefname{subappendix}{Appendix}{Appendices}
\usepackage{placeins}%
\usepackage{algorithm}
\usepackage{algorithmic}

\theoremstyle{plain}
\newtheorem{theorem}{Theorem}
\newtheorem*{theoremrestated}{Theorem~\ref{thm:duel}}%
\newtheorem{proposition}{Proposition}

\theoremstyle{definition}
\newtheorem{definition}{Definition}

\theoremstyle{remark}

\usepackage{turok-style}%

\definecolor{duelnavy}{HTML}{24456E}%
\definecolor{duelteal}{HTML}{2B7A72}%
\definecolor{duelrust}{HTML}{9A5033}%
\definecolor{duelmist}{HTML}{F1F4F8}%
\newcommand{\cmark}{\textcolor{duelteal}{\ding{51}}}%
\newcommand{\xmark}{\textcolor{duelrust}{\ding{55}}}
\hypersetup{colorlinks=true, linkcolor=duelnavy, citecolor=duelteal,
            urlcolor=duelrust, filecolor=duelrust}

\tcbuselibrary{skins,breakable}%
\tcbset{stmtbox/.style={enhanced, breakable, sharp corners, colback=duelmist, colframe=duelnavy,
  boxrule=0pt, leftrule=2pt, left=6pt, right=0pt, top=2.5pt, bottom=2.5pt, boxsep=0pt,
  grow to left by=8pt, before skip=3pt plus 1pt, after skip=3pt plus 1pt, enforce breakable}}
\tcolorboxenvironment{theorem}{stmtbox}
\tcolorboxenvironment{theoremrestated}{stmtbox}
\tcolorboxenvironment{proposition}{stmtbox}
\tcolorboxenvironment{definition}{stmtbox}
\definecolor{duelash}{HTML}{9AA3AE}
\tcbset{notebox/.style={stmtbox, colback=white, colframe=duelash, leftrule=1pt}}
\tcolorboxenvironment{assumption}{notebox}
\tcolorboxenvironment{remark}{notebox}
\newtcolorbox{promptbox}[1][]{enhanced, breakable, colback=duelmist, colframe=duelnavy!35!white,
  boxrule=0.4pt, arc=2pt, left=7pt, right=7pt, top=5pt, bottom=5pt,
  fonttitle=\bfseries\small, coltitle=duelnavy, colbacktitle=duelmist, titlerule=0.4pt,
  fontupper=\small, #1}
\newtcolorbox{samplebox}[1][]{enhanced, breakable, colback=duelmist, colframe=duelteal!45!white,
  boxrule=0.4pt, arc=2pt, left=7pt, right=7pt, top=5pt, bottom=5pt,
  fonttitle=\bfseries\small, coltitle=duelteal!80!black, colbacktitle=duelmist, titlerule=0.4pt,
  fontupper=\small, #1}

\graphicspath{{figures/}{./}}

\title{\textsc{DUEL}: Exact Likelihood Reshapes Masked\\ Diffusion Scaling Laws and Beyond}
\author{%
  Gilad Turok \\
  Department of Computer Science \\
  Cornell University \\
  \texttt{gt345@cornell.edu} \\
  \And
  Yair Schiff \\
  Department of Computer Science \\
  Cornell University \\
  \And
  Chris De Sa \\
  Department of Computer Science \\
  Cornell University \\
  \And
  Volodymyr Kuleshov \\
  Department of Computer Science \\
  Cornell University \\
}

\ifpreprint
\iclrfinalcopy%
\makeatletter
\RenewDocumentEnvironment{abstract}{+b}{\gdef\duel@abstract{#1}}{}%
\renewcommand{\maketitle}{%
  \begin{tcolorbox}[enhanced, colback=duelmist, frame hidden, arc=10pt,
                    left=0.5cm, right=0.5cm, top=0.3cm, bottom=0.3cm, before skip=0pt, after skip=8pt]
    \setlength{\parindent}{0pt}\centering
    {\LARGE\sc \@title\par}\vspace{0.5em}
    {\bfseries Gilad Turok, Yair Schiff, Chris De Sa$^{*}$, Volodymyr Kuleshov$^{*}$\par}\vspace{0.4em}
    {Cornell University\par}\vspace{0.5em}
    {\leftskip=0pt \rightskip=0pt \parfillskip=0pt plus 1fil \duel@abstract\par}\vspace{0.8em}
    {\raggedright\small\textbf{Website:} \href{https://giladturok.github.io/duel/}{giladturok.github.io/duel}\qquad
     \textbf{Code:} \href{https://github.com/kuleshov-group/duel}{github.com/kuleshov-group/duel}\par}
  \end{tcolorbox}%
  \begingroup\renewcommand{\thefootnote}{}\footnotetext{%
    $^{*}$Joint senior authorship. Correspondence to \texttt{gt345@cornell.edu}.}\endgroup
  \lhead{Preprint. Under review.}}
\makeatother
\hypersetup{pdftitle={DUEL: Exact Likelihood Reshapes Masked Diffusion Scaling Laws and Beyond},
  pdfauthor={Gilad Turok, Yair Schiff, Chris De Sa, Volodymyr Kuleshov}}
\fi
\begin{document}
\ifpreprint
  \begin{abstract}
Masked diffusion models (MDMs) are trained to predict tokens in every order but generate in one. 
As such, MDM likelihood is an intractable sum over all orders.
The evidence lower bound (ELBO) is used as a proxy for likelihood, but it is a \emph{loose bound} and assumes positions are unmasked uniformly at \emph{random}.
Frontier MDMs, however, select positions with \emph{deterministic} rules such as confidence.
To resolve this mismatch we introduce \textsc{DUEL}, a framework that unifies common deterministic MDM samplers.
\textbf{\textsc{DUEL} computes the likelihood \emph{exactly}, under the sampler that is \emph{actually used}} during generation.
\textsc{DUEL} is a simple plug-and-play algorithm applicable to any pretrained MDM with no retraining. 
To compute the likelihood of a text sequence, simply replay generation, reveal the true tokens instead of sampled ones, and sum their log-probabilities.
Evaluating with \textsc{DUEL} instead of the ELBO changes what we conclude about MDMs.
The perplexity gap between MDMs and autoregressive (AR) models shrinks by up to 32\% in-domain and 82\% zero-shot.
In scaling laws, the reported $16\times$ MDM--AR compute gap falls to $9\times$ under \textsc{DUEL}, showing over $\nicefrac{1}{3}$ of it is an evaluation artifact.
Scoring multiple-choice answers with likelihood via \textsc{DUEL} instead of the ELBO improves LLaDA-8B accuracy on all six tasks, by up to 9.7 points.
\textsc{DUEL} also selects samplers correctly: it matches a gold-standard LLM judge on \emph{every} decided comparison, which generative perplexity and MAUVE fail to do.
\textsc{DUEL} reveals that MDMs are more competitive with AR models than the ELBO suggests.

\end{abstract}
  \maketitle
\fi
\ifpreprint
\par\noindent
\begin{minipage}{\textwidth}
\centering
\includegraphics[width=\linewidth]{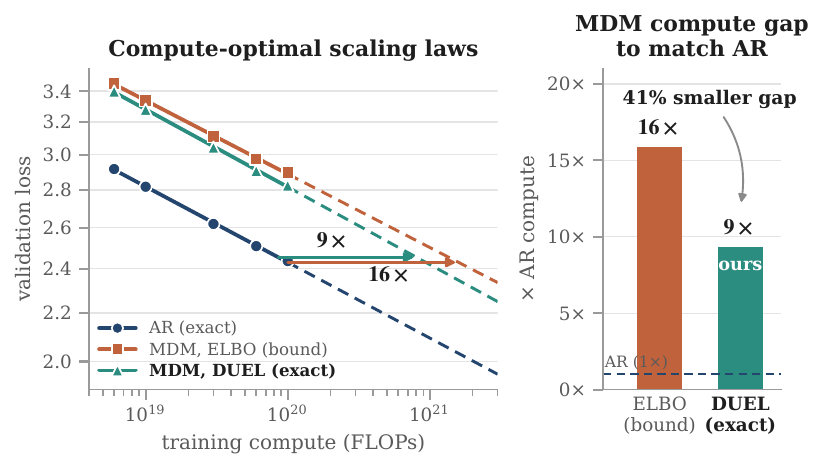}
\captionof{figure}{\textbf{\textsc{DUEL} shrinks the MDM--AR compute gap from $\mathbf{16}\times$
to $\mathbf{9}\times$.} Both MDM curves are the same SMDM checkpoints~\citep{nie2024scaling}, scored with the ELBO
\emph{bound} or our \emph{exact} \textsc{DUEL} likelihood. \emph{Left:} compute-optimal frontiers fitted over
$6$--$100 \times 10^{18}$ FLOPs (dashed: extrapolated); arrows show how much more compute an MDM
needs to match AR's loss.
\emph{Right:} the same MDM--AR compute gap at $10^{20}$ FLOPs, $15.9\times$ (ELBO) vs.\ $9.4\times$ (\textsc{DUEL}). Loss falls slowly with compute
(log $x$-axis), so \textsc{DUEL}'s small loss drop removes $41\%$ of the gap.
Full numbers in \cref{tab:scaling}.}
\label{fig:scaling}
\end{minipage}
\fi

\section{Introduction}
\label{sec:introduction}

Masked diffusion models (MDMs) bring diffusion to
text~\citep{austin2021structured,sahoo2024simple,shi2024simplified,lou2024discrete} and now scale to
billions of parameters with strong benchmark
results~\citep{nie2025large,bie2025llada2,bie2026llada21,labs2025mercury,song2025seed,diffusiongemma2026}. Unlike
autoregressive (AR) models, an MDM can generate several tokens per forward pass, in
any order, and can be steered by guidance~\citep{schiff2024simple}.

MDM likelihood, however, is harder to compute than that of AR models.
Following \citet{kim2025train}, an AR model is \emph{order-aware}: its
generation order is fixed as left-to-right, so its likelihood factorizes along
this order and is exactly computable. An MDM is
\emph{order-agnostic}~\citep{uria2014deep,kim2025train}: it is trained to reveal tokens
in every order but generates in only one. Its likelihood therefore marginalizes
over all unmasking orders, an intractable sum with more than $L!$ terms.

In practice, the evidence lower bound (ELBO) is used as a tractable proxy for
MDM likelihood (and also serves as its training objective). This proxy has two
problems. First, it is only a lower bound: by Jensen's inequality, the expected
log-likelihood over unmasking orders lies below the log of the expected
likelihood.
Second, it assumes that tokens are unmasked uniformly at
\emph{random} during generation. However, open frontier MDMs such as LLaDA~\citep{nie2025large}, Dream~\citep{ye2025dream},
LLaDA2.0~\citep{bie2025llada2}, SDAR~\citep{cheng2025sdar} and Nemotron-Labs-Diffusion~\citep{fu2026nemotron}
choose positions by a \emph{deterministic} rule, such as greedily revealing the most confident positions first.
The ELBO thus bounds the likelihood of a sampler that is rarely run.

We introduce the Deterministic Unmasking Exact Likelihood (\textsc{DUEL}) framework to
resolve the shortcomings of the ELBO. A \textsc{DUEL} sampler pairs a pretrained denoiser with a deterministic unmasking
rule to generate text. The rule fixes only which \emph{positions} to reveal; the tokens at those
positions are still sampled from the denoiser. Likelihood normally
marginalizes over all unmasking orders, a sum with more than $L!$ terms. 
However, a deterministic rule collapses this sum to a
single term that mirrors the generation procedure. 
To compute the likelihood of a text sequence, simply replay generation, revealing the
true tokens instead of sampled ones, and sum the log-probabilities from the denoiser.
\textsc{DUEL} is a plug-and-play evaluation
methodology: it computes proper likelihood for a standard MDM under its actual
test-time policy, with no retraining or architectural changes, and it corrects
how MDMs are measured, not the models themselves.

Accurate likelihood is central to evaluating MDMs: it shapes our understanding
of their capabilities, shortcomings, and cost, especially in comparison to
AR models. Evaluating MDMs with \textsc{DUEL} instead of the ELBO changes
that understanding, revealing MDMs to be more competitive than previously
thought along several axes. Our contributions are as follows:
\begin{itemize}[leftmargin=1.2em,itemsep=1pt,topsep=2pt]
  \item \textbf{\textsc{DUEL}.} We formalize \textsc{DUEL} samplers, unifying popular deterministic
  MDM samplers, and prove their likelihood is exact and efficiently computable (\cref{thm:duel}).
  \item \textbf{Perplexity gap.} The MDM--AR perplexity gap shrinks by 12--32\%
  in-domain and 22--82\% zero-shot across eight datasets and four models (SEDD,
  MDLM, BD3-LM, LLaDA-8B).
  \item \textbf{Scaling laws.} The compute an
  MDM needs to match AR loss falls from roughly $16\times$ under the ELBO to
  $9\times$ under \textsc{DUEL}, closing 31--41\% of the gap at every compute-optimal FLOP budget.
  \item \textbf{Multiple-choice benchmarks.} Scoring answers with \textsc{DUEL}, not the ELBO,
  lifts LLaDA-8B accuracy on all six tasks, by up to $9.7$ points ($2.5$ and $4.2$
  on average under two samplers).
  \item \textbf{Sampler selection.} \textsc{DUEL} agrees
  with a gold-standard LLM judge on sampler quality on $14/14$ decided comparisons, against $8/14$ for
  generative perplexity and $12/14$ for MAUVE.
\end{itemize}

\section{Preliminaries}
\label{sec:background}

\subsection{Masked diffusion models}
\label{sec:latent}

\paragraph{Data and masks.}
A sequence is $\bx=(x^1,\dots,x^L)\in\vocab^L$ over a vocabulary
$\vocab$. During generation only some positions are known; for a set
$S\subseteq[L]$ of positions we write $\rev{S}\triangleq(x^\ell)_{\ell\in S}$ for
the tokens at $S$ and call it a \emph{partially revealed} sequence. It
records $S$, hence also the masked positions $[L]\setminus S$.

\begin{figure}[t]
\centering
\includegraphics[width=\linewidth]{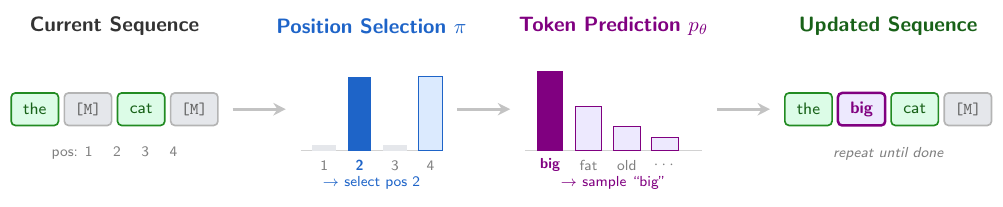}
\caption{\textbf{One step of MDM generation.}
Given a partially masked sequence, the unmasking policy $\pol$ performs
\emph{position selection} (choosing position~2), and the denoiser $\pth$ performs
\emph{token prediction} at that position (sampling ``big''). This repeats until
no masked positions remain.}
\label{fig:mdm-generation}
\end{figure}

\paragraph{Forward process.}
The forward process $q$ masks each position independently with probability
$1-\alpha_t$, where $\alpha_0=1$ and $\alpha_1=0$ (none masked to all masked):
$q_t(\bz_t^{\ell}\mid \bx)=\Cat\big(\alpha_t\,\e_{x^\ell}+(1-\alpha_t)\,\e_{\mathrm{m}}\big)$.
Masking is absorbing: $\bz_t$ is a partially revealed sequence with revealed
set $S(\bz_t)$ shrinking in $t$.

\paragraph{Reverse process.}
Given the clean sequence, the forward process can be reversed exactly. From
time $t$ to $s<t$, the posterior $q_{s\mid t}(\bz_s\mid\bz_t,\bx)$ keeps every
revealed position and reveals each masked position with probability
$(\alpha_s-\alpha_t)/(1-\alpha_t)$, at its clean value~\citep{sahoo2024simple,shi2024simplified}.
At generation time $\bx$ is unknown. A network, the \emph{denoiser}, takes a
partially revealed sequence and outputs, for every masked position $\ell\notin S$,
a categorical distribution $\pth(x^\ell\mid\rev{S})$ over the vocabulary predicting
the clean token there. The generative model substitutes these predictions for the
unknown clean sequence, $\pth(\bz_s\mid\bz_t)=q_{s\mid t}(\bz_s\mid\bz_t,\,\bx\leftarrow \pth)$,
and runs from $t=1$ (fully masked) to $t=0$.

\paragraph{Training objective.}
The likelihood $\pth(\bx)$ marginalizes over the latent trajectory and is
intractable. MDMs are trained on its evidence lower bound (ELBO), a weighted
cross-entropy at masked positions~\citep{sahoo2024simple,shi2024simplified,ou2024your}:
\begin{equation}
  -\log \pth(\bx)\ \le\ \elbo(\theta;\bx)\ \triangleq\ \E_{t,\,\bz_t}\Big[\,w_t \sum_{\ell\notin S(\bz_t)} -\log \pth\big(x^\ell\mid \rev{S(\bz_t)}\big)\Big],
  \qquad w_t=\frac{-\alpha_t'}{1-\alpha_t},
  \label{eq:mdm-loss}
\end{equation}
with $t\sim\mathrm{Unif}[0,1]$ and $\bz_t\sim q_t(\cdot\mid\bx)$. One denoiser is
trained on every masking level at once.

\subsection{Masked diffusion as any-order autoregression}
\label{sec:anyorder}

The diffusion formulation of \cref{sec:latent} hides an alternative description:
generation reveals the positions of $\bx$ one or a few at a time. 
We make the unmasking order explicit by framing MDMs as any-order autoregressive models
(AO-ARMs), a class originating in order-agnostic NADE~\citep{uria2014deep} and shown equivalent to MDMs~\citep{hoogeboom2021autoregressive,ou2024your,kim2025train}.

\begin{definition}[Ordered partition]
\label{def:partition}
An \emph{ordered partition} of $[L]=\{1,\dots,L\}$ is a tuple
$\sigma=(\sigma_1,\dots,\sigma_T)$ of non-empty, disjoint subsets with
$\bigcup_{t=1}^{T}\sigma_t=[L]$ and $\sigma_t\cap\sigma_{t'}=\emptyset$.
\end{definition}

\paragraph{Ordered partitions.}
An ordered partition describes the order in which an MDM unmasks tokens: $\sigma_t$
is the set of positions revealed at step $t$. We write
$\sigma_{<t}=\sigma_1\cup\dots\cup\sigma_{t-1}$ for the positions revealed before
step $t$, so that the sampler's state is the partially revealed
sequence $\rev{\sigma_{<t}}$. Ordered partitions capture both sequential and parallel
unmasking, a distinguishing feature of MDMs over AR models. When
every part of the ordered partition is a singleton we recover sequential unmasking with $T=L$ steps and
$\sigma$ is a permutation; otherwise we have parallel unmasking with $T<L$ steps,
possible because the denoiser predicts \emph{every} masked position in one forward pass.

\paragraph{Joint factorization.}
With the order made explicit via $\sigma$, an MDM sampler alternates two operations visualized in \cref{fig:mdm-generation}. 
It \emph{selects} which masked positions to reveal next, and it \emph{predicts} the
tokens at those positions. Position selection is governed by an \emph{unmasking
policy} $\pol(\sigma_t\mid\rev{\sigma_{<t}})$, a distribution over non-empty
subsets of the still-masked positions. The policy may depend on the tokens
revealed so far, $\rev{\sigma_{<t}}$, and therefore also on the denoiser's
predictions computed from them, e.g.\ their confidence. Token prediction draws a token at each selected position $\ell\in\sigma_t$ from the denoiser
$\pth(\cdot\mid\rev{\sigma_{<t}})$. The probability of generating $\bx$ via $\sigma$ factorizes accordingly (\cref{prop:joint}):
\begin{equation}
  \pth(\bx,\sigma)
  = \prod_{t=1}^{T}\
    \underbrace{\pol\big(\sigma_t\mid \rev{\sigma_{<t}}\big)}_{\text{position selection}}
    \ \cdot\
    \underbrace{\prod_{\ell\in\sigma_t} \pth\big(x^{\ell}\mid \rev{\sigma_{<t}}\big)}_{\text{token prediction}} .
  \label{eq:joint}
\end{equation}
The product over $\sigma_t$ is parallel prediction: one forward pass gives a distribution at every masked
position, and the tokens at $\sigma_t$ are drawn independently (\cref{app:joint}). 

\begin{table}[t]
\centering\small
\setlength{\tabcolsep}{5pt}
\renewcommand{\arraystretch}{1.2}
\caption{\textbf{The unmasking rules}, as functions of the revealed tokens.
$\mset$ is the set of masked positions, $k' = \min(k,|\mset|)$, $\mathrm{top}_j$ breaks ties by
lowest index, and $\langle\cdot\rangle$ falls back to the argmax (\cref{eq:fallback}).
$\mu$ and $\nu$ are thresholds; $\hat{P}_\ell$ is the previous prediction.
All five are deterministic, so \cref{thm:duel} applies.}
\label{tab:rules}
\begin{tabular*}{\linewidth}{@{\extracolsep{\fill}}llc@{}}
\toprule
\textbf{Unmasking rule $F$} & \textbf{Positions revealed next $F(\rev{\sigma_{<t}})$} & \textbf{Tokens per step $|F|$} \\
\midrule
Left-to-right
  & $\{\text{the } k' \text{ smallest } \ell \in \mset\}$ & $k'$ \\
Greedy confidence~\citep{nie2025large}
  & $\mathrm{top}_{k'}\big(P^{(1)}_\ell\big)$ & $k'$ \\
Probability margin~\citep{kim2025train}
  & $\mathrm{top}_{k'}\big(P^{(1)}_\ell - P^{(2)}_\ell\big)$ & $k'$ \\
Confidence threshold~\citep{wu2025fast}
  & $\big\langle \{\ell \in \mset : P^{(1)}_\ell \geq \mu\} \big\rangle$ & adaptive \\
KLASS~\citep{kim2025klass}
  & $\big\langle \{\ell \in \mset : P^{(1)}_\ell \geq \mu, \;
     \KL(\hat{P}_\ell \,\|\, P_\ell) \leq \nu\} \big\rangle$ & adaptive \\
\bottomrule
\end{tabular*}
\end{table}

\paragraph{Induced distribution.}
Summing over all orders $\sigma$ gives the distribution a policy $\pol$ generates from:
\begin{equation}
  \ppol{\pol}(\bx)\ =\ \sum_{\sigma} \pth(\bx,\sigma).
  \label{eq:marginal}
\end{equation}
\emph{An MDM thus implicitly defines a family of probability distributions over
text, $\{\ppol{\pol}\}_{\pol}$, indexed by the unmasking policy $\pol$}. The same
trained network yields a different distribution for each policy. The sum over
$\sigma$ is what makes an MDM \emph{order-agnostic}, and it is why its likelihood
appears intractable: there are more than $L!$ ordered partitions, and each term needs its own $T$ forward passes because the denoiser
conditions on a different context under every order.

\begin{figure}[t]
\centering
\begin{minipage}[t]{0.485\linewidth}
\begin{algorithm}[H]
\caption{\textsc{DUEL} sampling}
\label{alg:sample}
\begin{algorithmic}[1]
\REQUIRE denoiser $\pth$, rule $F$
\STATE $S \gets \emptyset$ \hfill $\triangleright$ \emph{nothing revealed}
\WHILE{masked tokens remain}
  \STATE $\mP \gets \pth(\cdot \mid \rev{S})$ \hfill $\triangleright$ \emph{one forward pass}
  \STATE $I \gets F(\rev{S})$ \hfill $\triangleright$ \emph{positions to reveal}
  \FOR{$\ell \in I$}
    \STATE \hlsample{x^\ell \sim \Cat(P_\ell)} \hfill $\triangleright$ \textcolor{algsampleink}{\emph{sample a token}}
  \ENDFOR
  \STATE $S \gets S \cup I$
\ENDWHILE
\STATE \textbf{return} $\bx$ \hfill $\triangleright$ \emph{a sample from} $\ppol{\pF}$
\end{algorithmic}
\end{algorithm}
\end{minipage}\hfill
\begin{minipage}[t]{0.485\linewidth}
\begin{algorithm}[H]
\caption{\textsc{DUEL} exact log-likelihood}
\label{alg:ll}
\begin{algorithmic}[1]
\REQUIRE sequence $\bx$, denoiser $\pth$, rule $F$
\STATE $S \gets \emptyset$, $\mathit{ll} \gets 0$ \hfill $\triangleright$ \emph{nothing revealed}
\WHILE{masked tokens remain}
  \STATE $\mP \gets \pth(\cdot \mid \rev{S})$ \hfill $\triangleright$ \emph{one forward pass}
  \STATE $I \gets F(\rev{S})$ \hfill $\triangleright$ \emph{same rule, same order}
  \FOR{$\ell \in I$}
    \STATE \hlscore{\mathit{ll} \mathrel{+}= \log P_\ell[x^\ell]} \hfill $\triangleright$ \textcolor{algscoreink}{\emph{score the truth}}
  \ENDFOR
  \STATE $S \gets S \cup I$
\ENDWHILE
\STATE \textbf{return} $\mathit{ll}$ \hfill $\triangleright$ $\log \ppol{\pF}(\bx)$, exact
\end{algorithmic}
\end{algorithm}
\end{minipage}
\caption{\textbf{Sampling and scoring are the same loop.} Same rule, same
revealed sets, same $T$ forward passes --- only the highlighted lines
differ. $P_\ell$ is the denoiser's predicted distribution over the vocabulary
at position $\ell$, so $P_\ell[x^\ell]$ is the probability it gives the true token
there.}
\label{fig:algorithms}
\end{figure}

\paragraph{Connection to MDM training.}
The diffusion model of \cref{sec:latent} is one member of this family. Run
ancestrally, its reverse process reveals a uniformly random subset of the masked
positions at each step, regardless of the revealed tokens. It is therefore
the sampler of \cref{eq:marginal} under the \emph{uniform} policy $\punif$, which
gives every permutation probability $1/L!$ (\cref{app:loss-equals-avg}), and its
likelihood is $\pth(\bx)=\ppol{\punif}(\bx)$. Its training objective, the
diffusion ELBO of \cref{eq:mdm-loss}, is the expected negative log-likelihood along
a uniformly random order (\cref{prop:loss-equals-avg};
\citealp{ou2024your,zheng2024masked,kim2025train}), which upper-bounds the
negative log-likelihood by Jensen's inequality applied to the sum in
\cref{eq:marginal}:
\begin{equation}
  -\log \ppol{\punif}(\bx)\ \le\ \E_{\sigma\sim\mathrm{Unif}(\Sym)}\Big[\sum_{t=1}^{L}-\log \pth\big(x^{\sigma_t}\mid\rev{\sigma_{<t}}\big)\Big].
  \label{eq:elbo-unif}
\end{equation}
An MDM is thus an AO-ARM trained under uniform random ordering. The ELBO's Monte Carlo estimate, the usual evaluation metric, inherits both
defects of \cref{sec:introduction}: it is only a bound, and it is taken
under $\punif$, a policy no large-scale MDM samples with. \Cref{sec:duel} removes
both.

\section{\textsc{DUEL}: Exact likelihood under deterministic unmasking}
\label{sec:duel}

Leading MDMs do not unmask at random, since random unmasking degrades sample quality~\citep{nie2025large}.
Instead, they generate with a deterministic \emph{unmasking rule} (\cref{tab:rules}): a function
that reads the denoiser's predictions on the revealed tokens and returns the positions to reveal next.

\begin{definition}[Unmasking rule]
\label{def:rule}
An \emph{unmasking rule} $F$ maps a partially revealed sequence $\rev{S}$ to a
non-empty set of its masked positions, $F(\rev{S})\subseteq[L]\setminus S$,
and is deterministic: the same input always yields the same set. It may use the
denoiser's predictions $\pth(\cdot\mid\rev{S})$, themselves a function of $\rev{S}$.
\end{definition}

\paragraph{Instantiations.}
\Cref{tab:rules} catalogs the rules we consider, with $P^{(1)}_\ell$ and
$P^{(2)}_\ell$ the highest and second-highest probabilities the denoiser assigns
at masked position $\ell$, and $k$ the number of positions revealed per step:
\emph{left-to-right} reveals the $k$ leftmost masked positions (autoregressive
when $k=1$); \emph{greedy confidence}~\citep{nie2025large} commits to the $k$ most
confident positions; \emph{probability margin}~\citep{kim2025train} prefers
positions where the gap between the top two tokens is largest;
\emph{confidence threshold}~\citep{wu2025fast} reveals every position with
$P^{(1)}_\ell\ge\mu$ (adaptive $k$, falling back to the most confident position);
and \emph{KLASS}~\citep{kim2025klass} adds a stability check via the KL divergence
between consecutive predictions $\hat{P}_\ell$ and $P_\ell$. Each instantiation is a deterministic function of $\rev{S}$ (KLASS of the trajectory, via its previous prediction;
\cref{app:rules}) and is thus an unmasking rule. They compose with further strategies such as
block-based sampling~\citep{arriola2025block} and entropy-bounded adaptive
parallelism~\citep{ben2025accelerated}. A pretrained denoiser
with any such rule selects positions deterministically, which we exploit (\cref{app:rules}).

\begin{definition}[\textsc{DUEL} sampler]
\label{def:duel}
A \emph{\textup{\textsc{DUEL}} sampler} is a pair $(\pth,F)$ of a pretrained denoiser and an unmasking rule. The rule $F$ induces the unmasking policy $\pF$ as a point mass:
\begin{equation}
  \pF\big(\sigma_t\mid\rev{\sigma_{<t}}\big)=\mathbf{1}\big[\sigma_t=F(\rev{\sigma_{<t}})\big].
  \label{eq:policy-F}
\end{equation}
\end{definition}

A \textsc{DUEL} sampler draws $\bx \sim \ppol{\pF}$ in \cref{eq:marginal} by the procedure of \cref{alg:sample}: starting fully masked, each step reveals the positions $F(\rev{S})$ with tokens drawn from
$\pth(\cdot\mid\rev{S})$, until none remain. 
Replaying the rule on $\bx$
instead produces a single ordered partition $\sigF(\bx)$, defined recursively by
$\sigF_t=F(\rev{\sigF_{<t}})$, on which $\pF$ is a point mass: every other order
must disagree with the rule at some step and receives probability zero. $\sigF(\bx)$ depends on the data, since the
rule reads the revealed tokens. It is found only by replaying the rule step by step.

\newcommand{\thmduelstatement}[1]{%
Let $(\pth,F)$ be a \textup{\textsc{DUEL}} sampler and let $\bx\in\vocab^L$ be any sequence. We can compute the \emph{exact} likelihood of $\bx$ under $\ppol{\pF}$ as:
\begin{equation}
  \log \ppol{\pF}(\bx)\ =\ \sum_{t=1}^{T}\ \sum_{\ell\in\sigF_t}\log \pth\big(x^{\ell}\mid\rev{\sigF_{<t}}\big),
  #1
\end{equation}
for the order $\sigF(\bx)$ that $F$ produces on $\bx$ and its number of steps $T$.
\Cref{alg:ll} implements the procedure with one forward pass per step, $T\le L$ passes in total.%
}

\begin{theorem}[\textsc{DUEL} exact likelihood]
\label{thm:duel}
\thmduelstatement{\label{eq:duel}}
\end{theorem}

\begin{proof}[Proof sketch]
Substituting the point-mass policy of \cref{eq:policy-F} into \cref{eq:joint,eq:marginal},
\begin{equation*}
  \ppol{\pF}(\bx)
  =\sum_{\sigma}\prod_{t}\mathbf{1}\big[\sigma_t=F(\rev{\sigma_{<t}})\big]
    \prod_{\ell\in\sigma_t}\pth\big(x^\ell\mid\rev{\sigma_{<t}}\big)
  =\prod_{t}\prod_{\ell\in\sigF_t}\pth\big(x^\ell\mid\rev{\sigF_{<t}}\big),
\end{equation*}
since the indicators vanish for every order except $\sigF(\bx)$. Full proof in
\cref{app:thm-proof}.
\end{proof}

The second equality is the key observation: the order-agnostic likelihood in \cref{eq:marginal} sums
over more than $L!$ orders, but a deterministic rule gives all except one of them zero probability,
and replaying the rule on $\bx$ finds that one.
Unlike the ELBO in \cref{eq:elbo-unif}, \textsc{DUEL} has no Jensen gap and scores the sampler actually used.
It costs the same as generation, one forward pass per step, so it is cheaper than a Monte Carlo
ELBO on short spans and costlier on long ones (\cref{app:cost}).

\section{Experiments}
\label{sec:experiments}

Replacing the ELBO with \textsc{DUEL} changes what we conclude about MDMs in four likelihood-based
evaluations: perplexity, scaling laws, multiple-choice benchmarks, and sampler selection.

\subsection{Perplexity}
\label{sec:heldout}

\paragraph{Setup.}
We test whether the ELBO inflates the perplexity gap between MDMs and AR models~\citep{nie2024scaling}.
We compare SEDD~\citep{lou2024discrete}, MDLM~\citep{sahoo2024simple}, and BD3-LM~\citep{arriola2025block}
against an AR baseline with matched parameters, compute, and data, and LLaDA-8B~\citep{nie2025large}
against Llama3-8B~\citep{grattafiori2024llama3} at scale.
Every MDM reveals one token per step ($k=1$) by greedy confidence within blocks:
$\Lp=4$ for SEDD and MDLM, $\Lp\in\{4,8,16\}$ for BD3-LM, and $\Lp=32$ for LLaDA.
We evaluate in-domain on OWT and LM1B (\cref{tab:heldout}),
zero-shot on PTB, Wikitext, LM1B, Lambada, AG News, PubMed, and ArXiv (\cref{tab:heldout-zs}),
and at 8B on Wikitext, Lambada, and AG News (\cref{tab:heldout-llada}).
See \cref{app:heldout}.

\paragraph{Metric.}
The \emph{perplexity gap} is $\Delta = \mathrm{PPL}_{\mathrm{MDM}} - \mathrm{PPL}_{\mathrm{AR}}$, where $\Delta_{\mathrm{ELBO}}$
computes $\mathrm{PPL}_{\mathrm{MDM}}$ with the ELBO and $\Delta_{\textsc{DUEL}}$ computes it with \textsc{DUEL}.
$\Delta_{\mathrm{ELBO}}$ inflates the perplexity gap on two axes---a loose bound and the wrong policy ($\punif$ instead of $\pF$)---both of which $\Delta_{\textsc{DUEL}}$ corrects.
The \emph{gap closed} is the share of $\Delta_{\mathrm{ELBO}}$ that \textsc{DUEL} eliminates, $(\Delta_{\mathrm{ELBO}} - \Delta_{\textsc{DUEL}})/\Delta_{\mathrm{ELBO}}$, in percent.

\begin{table}[t]
\centering
\caption{\textbf{Held-out and zero-shot perplexity.}
Gap closed (\%) is the share of the ELBO's MDM--AR perplexity gap that \textsc{DUEL} removes; ELBO
bound under $\punif$, \textsc{DUEL} exact under $\pF$, \textbf{bold} wins.}
\label{tab:heldout}
\small
\begin{minipage}[t]{0.47\linewidth}
\centering
\subcaption{Held-out OWT perplexity. AR baseline: 17.54.}
\label{tab:heldout-owt}
\setlength{\tabcolsep}{2pt}
\begin{tabular}{@{}lccc@{}}
\toprule
\textbf{Model} & \textbf{ELBO}$\downarrow$ & \textbf{\textsc{DUEL}}$\downarrow$ & \textbf{Gap Closed}$\uparrow$ \\
\midrule
SEDD             & $\leq$24.10 & \textbf{22.58} & 23.2 \\
MDLM             & $\leq$22.98 & \textbf{21.86} & 20.6 \\
BD3-LM $\Lp$=4   & $\leq$20.73 & \textbf{19.73} & 31.3 \\
BD3-LM $\Lp$=8   & $\leq$21.68 & \textbf{20.37} & 31.6 \\
BD3-LM $\Lp$=16  & $\leq$22.27 & \textbf{20.76} & 31.9 \\
\bottomrule
\end{tabular}

\vspace{0.9em}

\subcaption{Held-out LM1B perplexity. AR baseline: 22.83.}
\label{tab:heldout-lm1b}
\setlength{\tabcolsep}{2pt}
\begin{tabular}{@{}lccc@{}}
\toprule
\textbf{Model} & \textbf{ELBO}$\downarrow$ & \textbf{\textsc{DUEL}}$\downarrow$ & \textbf{Gap Closed}$\uparrow$ \\
\midrule
SEDD             & $\leq$32.68 & \textbf{31.51} & 11.9 \\
MDLM             & $\leq$31.78 & \textbf{29.03} & 30.7 \\
BD3-LM $\Lp$=4   & $\leq$28.23 & \textbf{26.97} & 23.3 \\
BD3-LM $\Lp$=8   & $\leq$29.83 & \textbf{27.85} & 28.3 \\
BD3-LM $\Lp$=16  & $\leq$30.60 & \textbf{28.39} & 28.4 \\
\bottomrule
\end{tabular}
\end{minipage}\hfill
\begin{minipage}[t]{0.49\linewidth}
\centering
\subcaption{Zero-shot gap closed (\%, $\uparrow$); trained on OWT, BD3-LM $\Lp$=4. Full PPL in \cref{tab:heldout-zs-full}.}
\label{tab:heldout-zs}
\begin{tabular*}{\linewidth}{@{\extracolsep{\fill}}lccc@{}}
\toprule
\textbf{Dataset} & \textbf{SEDD} & \textbf{MDLM} & \textbf{BD3-LM} \\
\midrule
PTB      & 31.3 & 34.3 & 81.8 \\
Wikitext & 40.8 & 28.5 & 31.4 \\
LM1B     & 22.1 & 29.0 & 29.8 \\
AG News  & 25.7 & 27.8 & 51.7 \\
\midrule
\textit{Average} & 30.0 & 29.9 & 48.7 \\
\bottomrule
\end{tabular*}

\vspace{0.9em}

\subcaption{Zero-shot perplexity ($\downarrow$), 8B models.}
\label{tab:heldout-llada}
\setlength{\tabcolsep}{2pt}
\begin{tabular*}{\linewidth}{@{\extracolsep{\fill}}llccc@{}}
\toprule
\textbf{Model} & \textbf{Estimator} & \textbf{Wikitext} & \textbf{Lambada} & \textbf{AG News} \\
\midrule
Llama3 & exact & 7.94 & 32.40 & 41.29 \\
\midrule
\multirow{2}{*}{LLaDA}
  & ELBO    & $\leq$15.29 & $\leq$39.04 & $\leq$85.17 \\
  & \textsc{DUEL} & \textbf{14.50} & \textbf{36.00} & \textbf{78.91} \\
\bottomrule
\end{tabular*}
\end{minipage}
\end{table}

\paragraph{Results.}
\textsc{DUEL} \textbf{closes 12--32\% of the MDM--AR gap in-domain and 22--82\% zero-shot}, on every model and corpus (\cref{tab:heldout-owt,tab:heldout-lm1b,tab:heldout-zs}; full perplexities in \cref{tab:heldout-zs-full}).
At 8B scale (\cref{tab:heldout-llada}), \textsc{DUEL} lowers LLaDA-8B perplexity by 5--8\% on Wikitext, Lambada, and AG News. Since LLaDA and Llama3 differ in training data and compute, we report raw perplexity instead of the gap closed.
Proper perplexity evaluation therefore \emph{substantially} closes the MDM--AR perplexity gap, and \textbf{the gain comes entirely from evaluation, not from any change to the models}.

\subsection{Scaling laws}
\label{sec:scaling}

\paragraph{Setup.}
\citet{nie2024scaling} conclude that an MDM needs ${\sim}16\times$ the training compute of an AR model to reach the same validation loss.
That comparison scores AR likelihood exactly but the MDM with the ELBO, a bound under the training policy $\punif$ rarely used during generation.
We redo their analysis with \textsc{DUEL} on their released AR and MDM IsoFLOP grids (non-embedding sizes 19M--1028M, budgets $C \in \{6, 10, 30, 60, 100\} \times 10^{18}$ FLOPs, $C = 6ND$), used unchanged.
On the same (up to 800) held-out SlimPajama sequences (\cref{app:scaling}) we compute, per checkpoint, the exact AR likelihood, the ELBO with $\mathrm{MC} = 128$ samples (their default), and \textsc{DUEL} at $k = 1$ under \emph{left-to-right} unmasking, the order closest to AR generation.
Full details in \cref{app:scaling}.

\begin{table}[t]
\centering
\captionof{table}{%
\textbf{The MDM--AR compute gap across SMDM's scaling law.} Compute-optimal
frontier on their IsoFLOP grid; loss in nats/token, gap the equal-loss multiplier
over AR. Both MDM blocks are the same checkpoints: the ELBO a bound under
$\punif$, \textsc{DUEL} exact (left-to-right); fits in \cref{tab:scaling-fits}.}
\label{tab:scaling}
\small\setlength{\tabcolsep}{11pt}
\begin{tabular}{@{}l r rr rr c@{}}
\toprule
& \multicolumn{1}{c}{\textbf{AR}} & \multicolumn{2}{c}{\textbf{MDM ELBO}} & \multicolumn{2}{c}{\textbf{MDM \textsc{DUEL}}} & \\
\cmidrule(lr){2-2}\cmidrule(lr){3-4}\cmidrule(lr){5-6}
\textbf{Compute (FLOPs)} & \textbf{loss} & \textbf{loss} & \textbf{gap} & \textbf{loss} & \textbf{gap} & \textbf{gap closed} \\
\midrule
$6 \times 10^{18}$   & 2.918 & 3.449 & $14.6\times$ & 3.398 & $10.0\times$ & $\mathbf{31\%}$ \\
$10 \times 10^{18}$  & 2.818 & 3.339 & $14.8\times$ & 3.279 & $9.9\times$  & $\mathbf{33\%}$ \\
$30 \times 10^{18}$  & 2.620 & 3.110 & $15.3\times$ & 3.044 & $9.6\times$  & $\mathbf{37\%}$ \\
$60 \times 10^{18}$  & 2.508 & 2.977 & $15.7\times$ & 2.908 & $9.5\times$  & $\mathbf{40\%}$ \\
$100 \times 10^{18}$ & 2.436 & 2.897 & $15.9\times$ & 2.825 & $9.4\times$  & $\mathbf{41\%}$ \\
\bottomrule
\end{tabular}

\vspace{1.2em}
\includegraphics[width=0.9\linewidth]{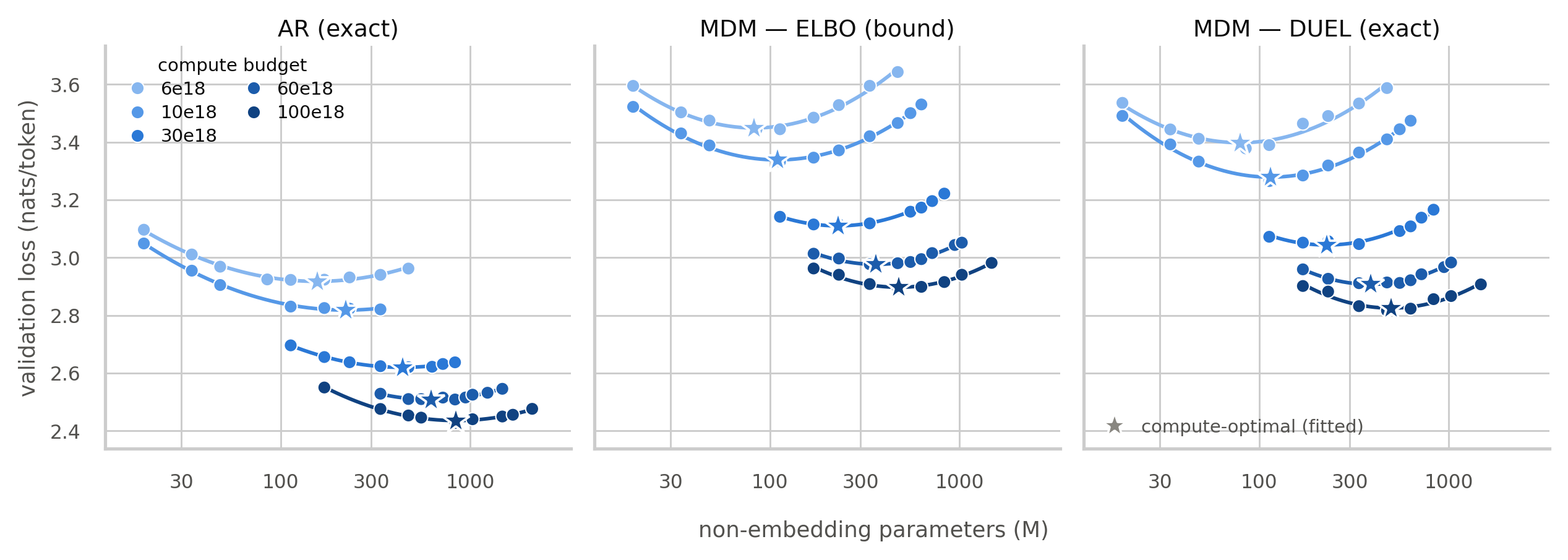}
\captionof{figure}{\textbf{IsoFLOP curves for AR, MDM ELBO and MDM \textsc{DUEL}.} Loss against non-embedding parameters $N$ at fixed budget $C$ ($C = 6ND$); stars mark the fitted vertices, i.e.\ the frontier in \cref{tab:scaling}.
The MDM panels are the \emph{same} checkpoints, scored by the ELBO ($\punif$, bound) and by \textsc{DUEL} left-to-right (exact); \textsc{DUEL} sits $0.05$--$0.07$ nats lower with a steeper frontier, hence $16\times \to 9\times$.}
\label{fig:isoflop}
\end{table}

\paragraph{Metric.}
The fitting procedure is SMDM's~\citep{nie2024scaling}, with only the $y$-axis changed.
Per budget we fit a quadratic of loss against $\log N$ and take its vertex as the compute-optimal $(N_{\mathrm{opt}}, L_{\mathrm{opt}})$; we then fit $L_{\mathrm{opt}} \propto C^{-a}$ across budgets by OLS in log-log space.
The equal-loss compute gap at budget $C$ solves $L_{\mathrm{MDM}}(x) = L_{\mathrm{AR}}(C)$ and reports $x / C$, the horizontal distance the arrows in \cref{fig:scaling} draw, inverting each frontier with \emph{its own} fitted exponent $a$ (\cref{app:scaling}).

\paragraph{Results.}
Our ELBO frontier (\cref{tab:scaling,fig:scaling}) recovers the published ${\sim}16\times$ MDM--AR compute gap, validating our replication of \citet{nie2024scaling}.
The exact \textsc{DUEL} likelihood then \textbf{drops the compute gap from $\mathbf{15}$--$\mathbf{16}\times$ to $\mathbf{9}$--$\mathbf{10}\times$, closing $\mathbf{31}$--$\mathbf{41\%}$ of it at every budget}.
The gap closed grows slightly with scale, though five budgets may not resolve this trend.
\Cref{fig:isoflop} shows why: \textsc{DUEL} lowers every MDM IsoFLOP curve while leaving the compute-optimal model size nearly unchanged (within $7\%$), so the frontier drops and the scaling exponent steepens slightly ($a = 0.0661$ vs.\ $0.0626$).
Since loss falls slowly as training scales, even the small improvement of \textsc{DUEL} over the ELBO saves many FLOPs.
The feasibility of large-scale MDM pretraining, especially compared to AR, depends on this compute gap, so accurate scaling laws need exact likelihood via \textsc{DUEL}.

\subsection{Multiple-choice benchmarks}
\label{sec:mc}

\paragraph{Setup.}
A multiple-choice benchmark scores each candidate answer by its likelihood and returns the highest, so the likelihood estimator directly sets the answer.
We evaluate LLaDA-8B on six such tasks, scored three ways: the ELBO at $\mathrm{MC}=128$ (the LLaDA default, a bound under the training policy $\punif$), and \textsc{DUEL} unmasking one token at a time ($k=1$) under both greedy-confidence and left-to-right unmasking, both exact under their own unmasking policy $\pF$.
We keep only tasks whose answer choices average about $18$ tokens or more, since shorter spans leave almost no ordering for the ELBO to get wrong.
Full details in \cref{app:mc}.

\begin{table}[t]
\centering\small
\caption{\textbf{Multiple-choice accuracy and calibration under the ELBO and \textsc{DUEL}.}
LLaDA-8B on six tasks. Sub-headers give the policy each estimator evaluates. Both \textsc{DUEL} columns win every row.}
\label{tab:mc}
\newcommand{\dlt}[1]{\,{\scriptsize($+#1$)}}
\begin{minipage}[t]{0.52\linewidth}
\centering
\subcaption{Accuracy (\%, $\uparrow$); gain over the ELBO in parentheses.}
\label{tab:mc-acc}
\setlength{\tabcolsep}{3.5pt}\renewcommand{\arraystretch}{1.1}
\begin{tabular}{lccc}
\toprule
 & \textbf{ELBO} & \multicolumn{2}{c}{\textbf{\textsc{DUEL}}} \\
\cmidrule(lr){2-2}\cmidrule(lr){3-4}
\textbf{Task} & \textbf{uniform} & \textbf{greedy} & \textbf{left-to-right} \\
\midrule
MuTual    & 50.1 & 51.6\dlt{1.5} & 51.7\dlt{1.6} \\
LogiQA    & 34.1 & 39.0\dlt{4.9} & 41.6\dlt{7.5} \\
LSAT-RC   & 55.4 & 60.2\dlt{4.8} & 65.1\dlt{9.7} \\
PIQA      & 74.8 & 76.1\dlt{1.3} & 76.7\dlt{1.9} \\
LSAT-LR   & 41.0 & 42.5\dlt{1.5} & 45.0\dlt{4.0} \\
HellaSwag & 70.7 & 71.4\dlt{0.7} & 71.1\dlt{0.4} \\
\midrule
\textit{Average} & 54.3 & 56.8\dlt{2.5} & 58.5\dlt{4.2} \\
\bottomrule
\end{tabular}
\end{minipage}\hfill
\begin{minipage}[t]{0.47\linewidth}
\centering
\subcaption{Calibration, six-task average. Per task in \cref{tab:mc-cal}.}
\label{tab:mc-calsum}
\setlength{\tabcolsep}{3.5pt}\renewcommand{\arraystretch}{1.1}
\begin{tabular}{lccc}
\toprule
 & \textbf{ELBO} & \multicolumn{2}{c}{\textbf{\textsc{DUEL}}} \\
\cmidrule(lr){2-2}\cmidrule(lr){3-4}
\textbf{Metric} & \textbf{uniform} & \textbf{greedy} & \textbf{left-to-right} \\
\midrule
Brier $\downarrow$ & 0.627 & 0.608 & 0.594 \\
NLL $\downarrow$   & 1.170 & 1.148 & 1.125 \\
AUROC $\uparrow$   & 0.597 & 0.630 & 0.644 \\
\bottomrule
\end{tabular}
\end{minipage}
\end{table}

\paragraph{Accuracy.}
We report accuracy for all three estimators, length-normalized (per character) on HellaSwag only, following the original LLaDA evaluation~\citep{nie2025large}.
\textbf{\textsc{DUEL} improves accuracy on all six tasks with no change to the model: $\mathbf{+2.5}$ points on average under
greedy confidence and $\mathbf{+4.2}$ under left-to-right} (\cref{tab:mc-acc}).
The largest gain is $+9.7$ points (LSAT-RC, left-to-right).

\paragraph{Calibration.}
\textbf{\textsc{DUEL} gives a better distribution over answers, not merely a better argmax.} 
Per-choice log-likelihoods are token log-probability sums, so a softmax over them is wildly overconfident and the estimators differ in scale as well as ranking.
Following \citet{guo2017calibration}, we temperature-scale per task and estimator by 2-fold cross-fitting. 
Both \textsc{DUEL} variants beat the ELBO on every task by Brier score and NLL, and on five of six by AUROC (\cref{tab:mc-calsum}).

\begin{table}[t]
\centering
\caption{\textbf{Agreement with the LLM judge.} Decided pairs on which each metric
ranks the samplers as the judge does, with all $12$ pairs per scale
in parentheses. Intervals in \cref{tab:judge-pairs}.}
\label{tab:judge-agree}
\small\setlength{\tabcolsep}{10pt}
\begin{tabular}{@{}lccc@{}}
\toprule
\textbf{Scale} & \textbf{\textsc{DUEL}} & \textbf{gen-ppl} & \textbf{MAUVE} \\
\midrule
$110$M (BD3-LM) & \textbf{7/7} (10/12) & 1/7 (4/12) & 5/7 (8/12) \\
$8$B (LLaDA)    & \textbf{7/7} (11/12) & 7/7 (11/12) & 7/7 (8/12) \\
\bottomrule
\end{tabular}
\end{table}

\subsection{Sampler selection}
\label{sec:judge}

\paragraph{Setup.}
An MDM's performance greatly depends on the sampler it generates with. 
Benchmarks can sometimes select samplers, but not always: small research models score near chance, and benchmarks fix target tasks in advance. A
likelihood is a good alternative, but the ELBO is invariant to the unmasking rule $F$---it
scores under $\punif$ whatever rule is run---so it cannot rank samplers.
In contrast, \textsc{DUEL} scores each rule under its own policy $\pF$, so it
can rank them.
We generate at two scales: $110$M BD3-LM~\citep{arriola2025block} 
(unconditional) and LLaDA-8B~\citep{nie2025large} (prefix-conditioned). We
compare left-to-right, greedy confidence, and probability margin, all under $\pF$,
at four matched budgets, revealing $k \in \{1,2,4,8\}$ tokens per pass
(confidence threshold is excluded, as its number of passes varies).
Full details in \cref{app:judge}.

\paragraph{Metric.}
We rank the samplers by \textsc{DUEL} perplexity and by two standard
sampler-sensitive quantities, generative perplexity (gen-ppl) and
MAUVE~\citep{pillutla2021mauve}, and take an LLM judge as ground truth.
The three read different text: \textsc{DUEL} scores \emph{held-out} OWT under the
policy (never the model's own samples), gen-ppl scores \emph{generated} text
under GPT-2 Large~\citep{radford2019language}, and MAUVE compares generated text to an
OWT reference. The judge, Claude Sonnet 5, is shown one generation from
each of two samplers at the same budget---blinded, in both A/B orders---and asked which is higher quality, weighing fluency,
grammaticality, coherence, and consistency. Each pair aggregates $n \approx 500$ comparisons and counts as \emph{decided}
when its win-rate CI excludes $0.5$ (\cref{tab:judge-pairs}).

\paragraph{Results.}
\textbf{\textsc{DUEL} agrees with the LLM judge on all $\mathbf{14}$ decided pairs, against
$\mathbf{8/14}$ for gen-ppl and $\mathbf{12/14}$ for MAUVE}
(\cref{tab:judge-agree}). Adding the $10$ undecided pairs, scored by point
direction, \textsc{DUEL} leads $21/24$ against $15/24$ and $16/24$.
\textsc{DUEL} perplexity tracks text quality directly: it
correlates with \emph{every} dimension the judge scores (fluency,
grammaticality, coherence, and consistency, with $\rho = -0.92$ to
$-0.99$), and the verdicts it agrees with are decided primarily by \emph{coherence}
($79\%$ at $110$M, $55\%$ at $8$B), not by surface grammar (\cref{app:judge}).

\paragraph{Why gen-ppl fails.}
Both \textsc{DUEL} and gen-ppl are cross-entropies, which decompose into entropy and KL divergence.
The two differ only in which distribution fills each slot: gen-ppl scores the
model's \emph{own} samples with an external scorer $\textcolor[HTML]{5E3A87}{q}$, while \textsc{DUEL} scores \emph{held-out}
data under the model. Writing $\textcolor{teal}{p_{\mathrm{data}}}$ for held-out OWT and $\textcolor{RawSienna}{\ppol{\pF}}$ for the text an MDM generates with sampler $\pF$, per token:
\begin{equation}
\log\text{gen-ppl}(F) = \Ent(\textcolor{RawSienna}{\ppol{\pF}}) + \KL(\textcolor{RawSienna}{\ppol{\pF}}\|\textcolor[HTML]{5E3A87}{q}),
\;\;
\log\text{\textsc{DUEL}-ppl}(F) = \Ent(\textcolor{teal}{p_{\mathrm{data}}}) + \KL(\textcolor{teal}{p_{\mathrm{data}}}\|\textcolor{RawSienna}{\ppol{\pF}}).
\label{eq:entropy-confound}
\end{equation}
In gen-ppl, the entropy $\Ent(\textcolor{RawSienna}{\ppol{\pF}})$ is the \emph{sampler's own}, so a policy can lower its
score by collapsing onto repetitive, low-entropy text without getting closer
to real data~\citep{pynadath2026generativefrontiersevaluationmatters}.
In \textsc{DUEL} the entropy is $\Ent(\textcolor{teal}{p_{\mathrm{data}}})$, a constant of the held-out data, so only the KL varies across samplers.
Thus a sampler can win on gen-ppl by producing \emph{worse} text via entropy collapse, but it can
win on \textsc{DUEL} only by producing text closer to $\textcolor{teal}{p_{\mathrm{data}}}$.
At $110$M and every parallel budget, left-to-right produces the most repetitive text (entropy $4.0$--$4.6$ vs.\ ${\approx}4.9$), which gen-ppl ranks \emph{best} and the LLM judge \emph{worst}.

\section{Discussion}
\label{sec:discussion}

\paragraph{Summary.}
An MDM is a family of distributions indexed by the unmasking policy $\pol$. The ELBO bounds the likelihood under a
rarely used uniform policy, while \textsc{DUEL} gives the exact likelihood under the policies frontier MDMs use.
Evaluating with \textsc{DUEL} instead of the ELBO shows MDMs performing better in perplexity, scaling, benchmarks, and sampler selection.

\paragraph{Related work.}
Existing routes to likelihood or non-uniform orderings in MDMs modify training or
architecture. Planner-aware training (PAPL~\citep{peng2025papl}, PUMA~\citep{kim2026puma})
and learned orders~\citep{peng2025path,wang2025learning} realign training with a
non-uniform policy but
still evaluate through a variational bound whose Jensen gap remains. Asymptotically exact likelihood has required the AO-ARM architecture of
Eso-LMs~\citep{sahoo2025esoteric} or continuous-time optimal-denoiser
assumptions~\citep{jeon2025information}. \citet{chen2025dultra} use per-trajectory
likelihoods for GRPO training. \textsc{DUEL} is complementary and
plug-and-play: exact likelihood for any standard MDM under any deterministic policy,
with no training or architectural change (\cref{app:related}).

\paragraph{Limitations.}
\textsc{DUEL} covers deterministic position selection. Token-level heuristics such as
nucleus or top-$k$ sampling are invisible to it, as they are to AR
perplexity. Stochastic policies are not point masses, and remasking samplers revisit revealed
tokens, so for both the sum over orders in \cref{eq:marginal} no longer collapses to a single, exactly computable term.
On long sequences \textsc{DUEL} costs more than the ELBO, which is worthwhile for final evaluations but not for monitoring training (\cref{app:cost}).

\subsection*{AI use statement}
In this work, we used generative AI tools (LLM coding assistants) to implement methods, to help
design experiments and provide feedback on the methodology, to interpret results, and to help
formulate mathematical claims and write proofs.
We have not used generative AI tools to generate synthetic data sets, develop theoretical
frameworks, or propose hypotheses, and translation, dataset cleaning, and qualitative data
analysis are not applicable to this work.
Additionally, we used generative AI tools to create and edit software code, figures, and tables,
to draft parts of the paper and edit it for readability, and to search for, summarize, and
format references to related literature.
Separately from this assistance, an LLM is part of the experimental design: Claude Sonnet~5 is
the judge in the sampler-selection study (\cref{sec:judge}), with prompts, blinding, and pairing
protocol given in \cref{app:judge}.
We have reviewed all AI-assisted work: the authors checked every AI-assisted claim and proof,
and verified AI-written code and the numbers it produced against logged runs.
We take responsibility for the final content of this work, including text, claims or
artifacts produced with the aid of generative AI.

\subsection*{Ethics statement}
This work proposes an evaluation method for an existing model class. It evaluates public
datasets and existing pretrained checkpoints, trains no models, collects no new data, involves no
human subjects, and releases no models or data with misuse risk. All datasets and models are
cited, used only for non-commercial research evaluation, and not redistributed; their
licenses or terms are listed in \cref{app:repro}. We do not foresee specific societal
harms from more accurate likelihood evaluation of masked diffusion models.

\subsection*{Reproducibility statement}
\textsc{DUEL} requires no training: every checkpoint is evaluated unchanged and is a public
release except the LM1B-trained 110M models, which were trained by the BD3-LM authors with
their public code and obtained from them; each is listed with its training budget and
evaluated block size in \cref{app:repro}.
The main result, \cref{thm:duel} (\cref{sec:duel}), is proved in full in \cref{app:thm-proof},
with the supporting propositions in \cref{app:background-proofs}.
\Cref{app:experiments} specifies each experiment's protocol (unmasking rules, Monte Carlo
budgets, sequence lengths, the IsoFLOP fitting procedure, and the judge prompts and pairing
protocol), \cref{app:cost} reports measured costs, and \cref{app:repro} lists
every checkpoint and dataset with its source and license. \ifpreprint Code for \textsc{DUEL} scoring and all evaluations is available at
\url{https://github.com/kuleshov-group/duel}.\fi

\newenvironment{ack}{\subsubsection*{Acknowledgments}}{}
\ificlrfinal\begin{ack}
This work was partially funded by the National Science Foundation under award CAREER 2145577,
and by the National Institutes of Health under award MIRA R35GM151243.
GT is partially supported by the NSF Graduate Research Fellowship under Grant No.\ DGE-213899.
\end{ack}
\fi

\bibliographystyle{iclr2027_conference}
\bibliography{bib/references}

\newpage
\setcounter{tocdepth}{2}
\tableofcontents
\clearpage

\appendix
\raggedbottom%
\setcounter{topnumber}{2}\renewcommand{\topfraction}{0.9}%
\crefalias{section}{appendix}\crefalias{subsection}{subappendix}%
\section{Related work}
\label{app:related}

Related work in five areas, each closing with how \textsc{DUEL} relates to it: masked diffusion,
any-order autoregression, likelihood estimation, unmasking policies, and MDM evaluation.

\paragraph{Masked diffusion.}
Discrete diffusion~\citep{hoogeboom2021argmax, austin2021structured, campbell2022continuous,
lou2024discrete} extends diffusion models~\citep{sohl2015deep, ho2020denoising, song2020score} to
discrete data. Among discrete corruption processes, masking proved the most effective, and
MDLM~\citep{sahoo2024simple} and MD4~\citep{shi2024simplified} simplified its objective to a weighted
masked-language-modeling loss. \citet{zheng2024masked} showed that the denoiser does not need the
diffusion time as input. MDMs now run at scale, from open models like LLaDA~\citep{nie2025large},
Dream~\citep{ye2025dream}, LLaDA2.0~\citep{bie2025llada2}, SDAR~\citep{cheng2025sdar} and
Nemotron-Labs-Diffusion~\citep{fu2026nemotron} to fast commercial systems~\citep{labs2025mercury,
song2025seed}, and \citet{kuleshov2026modern} describe how such models are built. Block
diffusion~\citep{arriola2025block} and E2D2~\citep{arriola2025encoder} diffuse within blocks
generated left to right, recovering KV caching. SMDM~\citep{nie2024scaling} fitted the
compute-optimal MDM scaling laws we revisit in \cref{sec:scaling}. \textsc{DUEL} applies to all of
these models unchanged: it needs only a trained denoiser and a deterministic rule.

\paragraph{Any-order autoregression.}
Any-order models predate diffusion: order-agnostic NADE~\citep{uria2014deep, uria2016neural},
which extends the fixed-order NADE~\citep{larochelle2011neural}, and XLNet~\citep{yang2019xlnet}
train one network on random orders. ARDMs~\citep{hoogeboom2021autoregressive} trained an any-order
autoregressive model with a diffusion-style objective, and \citet{ou2024your} and
\citet{zheng2024masked} showed the converse: the masked-diffusion objective \emph{is} the any-order
autoregressive objective, so every trained MDM is an any-order model
(\cref{prop:loss-equals-avg}). \citet{kim2025train} show that this generality has a cost: an MDM
must solve every infilling problem rather than only the $L$ left-to-right ones, and the hardest of
them dominate training. LO-ARM~\citep{wang2025learning} instead learns the generation order jointly
with the model, with results on molecular graphs. \textsc{DUEL} uses this equivalence
(\cref{sec:anyorder}) but scores the order the sampler actually runs rather than a random one.

\paragraph{Likelihood in masked diffusion.}
Standard MDM likelihoods are bounds: the objectives of MDLM, MD4 and \citet{ou2024your} all bound
the likelihood under the uniform policy $\punif$. Prior routes past them:
\begin{itemize}[leftmargin=*,itemsep=1pt,topsep=2pt]
\item \textbf{Change training.} PAPL~\citep{peng2025papl} derives an ELBO for planner-guided
  sampling but trains with a reweighted standard loss in practice; PUMA~\citep{kim2026puma} changes
  the forward masking process so that training masks resemble those of a test-time rule; set
  diffusion~\citep{arriola2026set} trains under a fixed distribution over orders. Each tightens or
  moves the bound to a non-uniform policy, but it stays a bound, and models trained any of these
  ways can be scored with \textsc{DUEL} as they are.
\item \textbf{Change the architecture.} Eso-LMs~\citep{sahoo2025esoteric} obtain an asymptotically
  exact, importance-weighted likelihood.
\item \textbf{Idealize the denoiser.} \citet{jeon2025information} derive an information-theoretic
  identity that makes the cross-entropy exact, but only for the optimal denoiser in continuous time,
  not for a trained model.
\item \textbf{Average over orders.} \citet{uria2014deep} and the importance-weighted estimator of
  \citet{burda2016iwae} average over orders drawn independently of $\bx$, whereas \textsc{DUEL}'s
  order depends on $\bx$ through the model's own predictions.
\item \textbf{Use \textsc{DUEL}'s quantity.} \citet{chen2025dultra} use per-trajectory token likelihoods of
  the kind \textsc{DUEL} computes, combined with the planner's selection probabilities, as the
  log-probability term in RL post-training.
\end{itemize}

\paragraph{Unmasking policies and samplers.}
MDM samplers differ mainly in how they choose which positions to reveal next: by confidence or
margin, by entropy, with a learned policy or sampling-time planner~\citep{peng2025path}, or on a
fixed schedule. Whenever such a choice is a deterministic function of the tokens revealed so far,
\textsc{DUEL} computes the exact likelihood under it (\cref{app:rules}); remasking
samplers, which revisit revealed positions, are out of scope.

\paragraph{Evaluating MDMs.}
Each use of exact likelihood here replaces a substitute with known failure modes.
\begin{itemize}[leftmargin=*,itemsep=1pt,topsep=2pt]
\item \textbf{Generative perplexity} scores samples under an external model, usually
  GPT-2~\citep{radford2019language}, and rewards low-entropy output because its entropy term is the
  sampler's own~\citep{pynadath2026generativefrontiersevaluationmatters}, as our judge study
  confirms (\cref{sec:judge}).
\item \textbf{Token entropy and MAUVE}~\citep{pillutla2021mauve} only partly correct this.
\item \textbf{Benchmark accuracy} is exposed to contamination~\citep{balloccu2024leak,
  white2024livebench}, and in likelihood-scored multiple choice it depends on the estimator itself
  (\cref{sec:mc}).
\end{itemize}
\textsc{DUEL} gives MDMs an exact likelihood under the sampler actually used, on the same footing as
an AR model's perplexity.

\clearpage
\section{Proofs}
\label{app:background-proofs}

Throughout, $\bx \in \vocab^L$ is a clean sequence, $S \subseteq [L]$ a set of
revealed positions, and $\rev{S}$ the partial sequence carrying the tokens of $\bx$ at
those positions together with their position indices; $\rev{S}$ therefore determines $S$.
An order $\sigma = (\sigma_1,\dots,\sigma_T)$ is an ordered partition of $[L]$
(\cref{def:partition}), with $\sigma_{<t} = \sigma_1 \cup \cdots \cup \sigma_{t-1}$.

\subsection{The joint factorization}
\label{app:joint}

Generation alternates position selection and token prediction. An \emph{unmasking
policy} $\pol(\sigma_t \mid \rev{\sigma_{<t}})$ is a conditional distribution over
non-empty sets of masked positions that reads the past only through the revealed tokens.
Starting fully masked, the sampler draws $\sigma_t \sim \pol(\cdot \mid \rev{\sigma_{<t}})$
at each step, then draws the token at each $\ell \in \sigma_t$ independently from the
denoiser's marginal $\pth(x^\ell \mid \rev{\sigma_{<t}})$, until no position is masked.

\paragraph{Within-step independence.}
Suppose step $t$ reveals $k = |\sigma_t|$ positions (fewer may remain at the last step;
\cref{app:rules}). One forward pass on
$\rev{\sigma_{<t}}$ returns a categorical $\pth(\cdot \mid \rev{\sigma_{<t}})$ at every masked
position, and \cref{alg:sample} draws each $x^\ell$, $\ell \in \sigma_t$, from its own
categorical. The joint token draw of the step is therefore exactly
\begin{equation*}
\pth\big(\rev{\sigma_t} \mid \rev{\sigma_{<t}}\big)
\;=\; \prod_{\ell \in \sigma_t} \pth\big(x^\ell \mid \rev{\sigma_{<t}}\big).
\end{equation*}
This describes the sampler; it does not approximate it. The sampler's distribution
simply carries no dependence among the $k$ tokens of a step, which is why changing $k$
changes the distribution rather than the accuracy of its likelihood. At $k = 1$ the
product has one factor and the statement is vacuous.

\begin{proposition}[Joint factorization]
\label{prop:joint}
This generative process defines a joint distribution over
$(\bx, \sigma)$,
\begin{equation}
\pth(\bx, \sigma) \;=\; \prod_{t=1}^{T}
  \pol\big(\sigma_t \mid \rev{\sigma_{<t}}\big) \prod_{\ell \in \sigma_t}
  \pth\big(x^\ell \mid \rev{\sigma_{<t}}\big),
  \qquad
  \ppol{\pol}(\bx) \;=\; \sum_{\sigma} \pth(\bx, \sigma),
\label{eq:joint-app}
\end{equation}
which is \cref{eq:joint,eq:marginal}.
\end{proposition}

\begin{proof}[Proof of \cref{prop:joint}]
A run of the sampler is the interleaved sequence
$(\sigma_1, \rev{\sigma_1}, \dots, \sigma_T, \rev{\sigma_T})$,
which determines and is determined by the pair $(\bx, \sigma)$. Write
$H_{t-1} = (\sigma_1, \rev{\sigma_1}, \dots, \sigma_{t-1}, \rev{\sigma_{t-1}})$ for the history
before step $t$. By the chain rule,
\begin{align*}
\pth(\bx, \sigma)
&= \prod_{t=1}^{T} \pth\big(\sigma_t, \rev{\sigma_t} \mid H_{t-1}\big) \\
&= \prod_{t=1}^{T} \pol\big(\sigma_t \mid H_{t-1}\big) \cdot
   \pth\big(\rev{\sigma_t} \mid \sigma_t, H_{t-1}\big) \\
&= \prod_{t=1}^{T} \pol\big(\sigma_t \mid \rev{\sigma_{<t}}\big)
   \prod_{\ell \in \sigma_t} \pth\big(x^\ell \mid \rev{\sigma_{<t}}\big).
\end{align*}
The second line splits each step into its position draw and its token draw. In the
third, the history is replaced by $\rev{\sigma_{<t}}$ because both the policy and the denoiser
read the past only through the revealed tokens; $\sigma_t$ is dropped because it only selects
which of the denoiser's outputs are read, not what they are; and the token factor
splits over $\ell \in \sigma_t$ by within-step independence. Every factor is a normalized
conditional and the process terminates, so $\pth(\cdot,\cdot)$ is a distribution and
$\ppol{\pol}$ is its $\bx$-marginal.
\end{proof}

\subsection{Proof of the exact-likelihood theorem}
\label{app:thm-proof}

\begin{theoremrestated}[\textsc{DUEL} exact likelihood]
\thmduelstatement{\tag{\ref{eq:duel}}}
\end{theoremrestated}

\begin{proof}[Proof of \cref{thm:duel}]
Replaying $F$ on $\bx$, $\sigF_t = F(\rev{\sigF_{<t}})$, gives one order $\sigF(\bx)$, and
every other order has probability zero: if $\sigma$ first differs from $\sigF(\bx)$ at step
$t$, the two share the revealed prefix, so $F(\rev{\sigma_{<t}}) = \sigF_t \neq \sigma_t$ and
$\pF(\sigma_t \mid \rev{\sigma_{<t}}) = 0$ by \cref{eq:policy-F}. Only $\sigF(\bx)$ survives the
sum in \cref{eq:joint-app}, and each of its policy factors equals $1$:
\begin{equation*}
\ppol{\pF}(\bx) \;=\; \sum_{\sigma} \pth(\bx, \sigma)
\;=\; \prod_{t=1}^{T} \underbrace{\pF\big(\sigF_t \mid \rev{\sigF_{<t}}\big)}_{=\,1}
\prod_{\ell \in \sigF_t} \pth\big(x^\ell \mid \rev{\sigF_{<t}}\big).
\end{equation*}
Taking logarithms gives \cref{eq:duel}. Each step reveals at least one position, so
$T \le L$, and one forward pass on $\rev{\sigF_{<t}}$ yields both $\sigF_t$ and its factors.
\end{proof}

The proof uses $F$ only through one property: orders that agree before step $t$
receive the same output at step $t$. \Cref{app:rules} uses this to admit stateful rules.
When a step reveals several positions, the theorem is exact for the sampler's own
distribution, whose tokens within a step are independent (\cref{app:joint}).

\subsection{The ELBO bounds the negative log-likelihood of the uniform policy}
\label{app:loss-equals-avg}

\paragraph{Setup.}
We take the sequential case, $\sigma \in \Sym$, and write the likelihood along one order as
\begin{equation}
\log \pth(\bx \mid \sigma) \;=\; \sum_{t=1}^{L} \log \pth\big(x^{\sigma_t} \mid \rev{\sigma_{<t}}\big).
\label{eq:chain-rule}
\end{equation}
The uniform policy $\punif(\sigma_t \mid \rev{\sigma_{<t}}) = 1/(L-t+1)$ gives every order
probability $1/L!$. In the continuous-time limit, where each step reveals at most one position,
it is the policy of the diffusion model run ancestrally, whose reverse process then reveals masked
positions in a uniformly random order regardless of the tokens~\citep{ou2024your, zheng2024masked}. As in \cref{sec:latent}, the noise schedule $\alpha$ decreases
strictly and differentiably from $1$ (clean) to $0$ (fully masked), and the denoiser
is time-free~\citep{sahoo2024simple, zheng2024masked}.

\begin{proposition}[The ELBO bounds the uniform member]
\label{prop:loss-equals-avg}
The ELBO $\elbo(\theta;\bx)$ of \cref{eq:mdm-loss} is the expected negative log-likelihood along a uniformly
random order, and therefore upper-bounds the negative log-likelihood under $\punif$:
\begin{equation}
\elbo(\theta; \bx)
\;=\; \E_{\sigma \sim \mathrm{Unif}(\Sym)}\big[ -\log \pth(\bx \mid \sigma) \big]
\;\geq\; -\log \ppol{\punif}(\bx),
\label{eq:loss-equals-avg}
\end{equation}
with equality if and only if $\pth(\bx \mid \sigma)$ is the same for every $\sigma \in \Sym$.
\end{proposition}

\begin{proof}[Proof of \cref{prop:loss-equals-avg}]
Write $\E_\sigma$ for $\E_{\sigma \sim \mathrm{Unif}(\Sym)}$. Then
\begin{align*}
-\log \ppol{\punif}(\bx)
&= -\log \E_{\sigma}\big[\pth(\bx \mid \sigma)\big] && \text{(i)} \\
&\le \E_{\sigma}\big[-\log \pth(\bx \mid \sigma)\big] && \text{(ii)} \\
&= \sum_{S \subsetneq [L]} \sum_{\ell \notin S} \frac{|S|!\,(L-|S|-1)!}{L!}
   \big(-\log \pth(x^\ell \mid \rev{S})\big) && \text{(iii)} \\
&= \elbo(\theta; \bx). && \text{(iv)}
\end{align*}
\begin{enumerate}[label=(\roman*), leftmargin=2em, itemsep=1pt, topsep=2pt]
\item \Cref{prop:joint} with $\pol = \punif$: $\punif$ reveals one position per step, so only
  permutations have nonzero probability, and each has policy product $1/L!$.
\item Jensen's inequality for the concave logarithm, with equality if and only if
  $\pth(\bx \mid \sigma)$ does not depend on $\sigma$.
\item Expand \cref{eq:chain-rule} and count: a pair $(S, \ell)$ with $|S| = n$ appears as
  $(\sigma_{<t}, \sigma_t)$ in exactly the $n!\,(L-n-1)!$ orders that place $S$ first and
  $\ell$ next.
\item Each position is masked independently with probability $1-\alpha$ and the denoiser
  is time-free, so averaging \cref{eq:mdm-loss} over masks and changing variables to $\alpha$
  gives the same weights, $\int_0^1 \alpha^n (1-\alpha)^{L-n-1}\,\mathrm{d}\alpha =
  n!\,(L-n-1)!/L!$ (\citealp[Prop.~3.1]{zheng2024masked}; \citealp[Prop.~2.1]{kim2025train}).
  \qedhere
\end{enumerate}
\end{proof}

The gap in step (ii) is the spread of $\log \pth(\bx \mid \sigma)$ over orders. Under $\pF$
the distribution over orders is a point mass on $\sigF(\bx)$, so the sum over orders has a single
term and the likelihood is exact (\cref{thm:duel}).

\section{Unmasking rules}
\label{app:rules}

Details for \cref{tab:rules}: the conventions that make each rule a deterministic function, how
stateful rules fit \cref{thm:duel}, which other rules are covered, and why the uniform policy is
never used to sample. The motivation is that the denoiser already tells the sampler where it is
confident, and revealing those positions first gives every later prediction a more reliable
context; revealing positions uniformly at random discards that information.

\paragraph{Notation and conventions.}
\Cref{tab:rules} (main text) defines the five rules formally. Fix a step and write
$\mset = \mset(\rev{\sigma_{<t}}) \subseteq [L]$ for the masked positions,
$\mP = \pth(\cdot \mid \rev{\sigma_{<t}})$ for the denoiser's output, and
$P^{(1)}_\ell \geq P^{(2)}_\ell$ for the largest and second-largest entries of $P_\ell$ at
position $\ell$. Three conventions make every rule a total, deterministic function, as
\cref{thm:duel} requires:
\begin{itemize}[leftmargin=*,itemsep=1pt,topsep=2pt]
\item \textbf{Fewer than $k$ left.} A fixed-$k$ rule returns $k' = \min(k, |\mset|)$ positions,
  since fewer than $k$ may remain at the last step of a sequence or block.
\item \textbf{Ties.} $\mathrm{top}_j(s)$ takes the $j$ positions with the largest scores and breaks
  ties by lowest index, so that $F$ is a function rather than a correspondence.
\item \textbf{Empty threshold.} Threshold rules select by a test that no position need pass, so
  they fall back to the single most confident position when the test selects nothing:
\begin{equation}
\langle \mathcal{A} \rangle \;\triangleq\;
\begin{cases}
\mathcal{A}, & \mathcal{A} \neq \emptyset, \\[2pt]
\mathrm{top}_1\big(P^{(1)}\big), & \mathcal{A} = \emptyset .
\end{cases}
\label{eq:fallback}
\end{equation}
\end{itemize}

\paragraph{Stateful rules such as KLASS are covered.}
KLASS compares the current prediction with the previous step's, so it is a function of the
trajectory rather than of $\rev{\sigma_{<t}}$ alone. Let the sampler carry state
$h_t = G(\rev{\sigma_{<t}}, h_{t-1})$ and let $F$ act on $(\rev{\sigma_{<t}}, h_{t-1})$. The
policy then reads the history $H_{t-1}$ of \cref{app:joint} rather than the revealed tokens alone,
so \cref{eq:joint} holds with $\pol(\sigma_t \mid H_{t-1})$ in place of
$\pol(\sigma_t \mid \rev{\sigma_{<t}})$: this is the second line of the proof of
\cref{prop:joint}, before the history is dropped. Since $h_{t-1}$ is a function of the trajectory
up to step $t-1$, two orders that agree on every step before $t$ receive the same output from $F$,
which is all the proof of \cref{thm:duel} uses. At $t=1$ we set $\hat{P}_\ell = P_\ell$, so KLASS
coincides with confidence threshold on the first step.

\paragraph{Other rules.}
\begin{itemize}[leftmargin=*,itemsep=1pt,topsep=2pt]
\item \textbf{Covered.} Entropy-based selection~\citep{ben2025accelerated}, schedules for how
  many positions to reveal at each step~\citep{luxembourg2025plan}, the confidence-, entropy- and
  margin-based selection rules analyzed by \citet{hayakawa2025demystifying}, and any of these
  composed with block decoding~\citep{arriola2025block}. All are deterministic functions of the
  revealed tokens, so \cref{thm:duel} applies to each without change.
\item \textbf{Not covered.} Rules that inject randomness into position selection, such as sampling
  positions in proportion to confidence, Gumbel-top-$k$ selection as in MaskGIT-style
  samplers~\citep{hayakawa2025demystifying}, and learned position policies that sample which
  positions to reveal~\citep{jazbec2025learning}; and remasking samplers~\citep{wang2025remasking,
  huang2025don} that revisit revealed positions, which breaks the ordered-partition structure of
  \cref{prop:joint}.
\end{itemize}

\paragraph{Nobody samples uniformly at random.}
Every large-scale MDM report we are aware of produces its headline results with a deterministic
rule: LLaDA~\citep{nie2025large}, Dream~\citep{ye2025dream}, LLaDA2.0~\citep{bie2025llada2},
SDAR~\citep{cheng2025sdar}, and Nemotron-Labs-Diffusion~\citep{fu2026nemotron} use confidence-, entropy-
or threshold-based rules. Where uniformly random unmasking appears at all, it is as an ablation
baseline that degrades quality~\citep{nie2025large, kim2025train}. The uniform policy $\punif$ is
what the training objective averages over, and the ELBO is invariant to the policy used at
inference, so standard evaluation never had to record which policy a model was run with. The policy
the ELBO scores is therefore one that no deployed sampler draws from, which is the mismatch
\cref{sec:duel} removes.

\clearpage
\section{Experiments}
\label{app:experiments}

\subsection{Perplexity}
\label{app:heldout}

\begin{table}[b]%
\centering
\caption{Zero-shot perplexity ($\downarrow$) for 110M models trained on OWT. AR is exact under the left-to-right order, the ELBO is a bound under $\punif$, and \textsc{DUEL} is exact under $\pF$. \textbf{Bold}: better of ELBO and \textsc{DUEL} per model--corpus pair.}
\label{tab:heldout-zs-full}
\small
\setlength{\tabcolsep}{4pt}
\begin{tabular}{@{}ll ccccccc@{}}
\toprule
\textbf{Model} & \textbf{Estimator} & \textbf{PTB} & \textbf{Wikitext} & \textbf{LM1B} & \textbf{Lambada} & \textbf{AG News} & \textbf{PubMed} & \textbf{ArXiv} \\
\midrule
AR & exact & 81.07 & 25.32 & 51.14 & 52.13 & 52.11 & 48.59 & 41.22 \\
\midrule
\multirow{2}{*}{SEDD}
  & ELBO    & $\leq$96.33 & $\leq$35.98 & $\leq$68.14 & $\leq$48.93 & $\leq$67.82 & $\leq$45.93 & $\leq$40.52 \\
  & \textsc{DUEL} & \textbf{91.56} & \textbf{31.63} & \textbf{64.39} & \textbf{46.01} & \textbf{63.78} & \textbf{39.31} & \textbf{35.48} \\
\midrule
\multirow{2}{*}{MDLM}
  & ELBO    & $\leq$90.96 & $\leq$33.22 & $\leq$64.94 & $\leq$48.29 & $\leq$62.78 & $\leq$43.64 & $\leq$38.35 \\
  & \textsc{DUEL} & \textbf{87.57} & \textbf{30.97} & \textbf{60.94} & 48.42 & \textbf{59.81} & \textbf{41.96} & \textbf{37.57} \\
\midrule
\multirow{2}{*}{\makecell[l]{BD3-LM\\$\Lp$=4}}
  & ELBO    & $\leq$96.81 & $\leq$31.31 & $\leq$60.88 & $\leq$50.03 & $\leq$61.67 & $\leq$42.52 & $\leq$39.20 \\
  & \textsc{DUEL} & \textbf{83.94} & \textbf{29.43} & \textbf{57.98} & \textbf{48.56} & \textbf{56.73} & \textbf{41.41} & \textbf{38.11} \\
\bottomrule
\end{tabular}
\end{table}

Details for \cref{sec:heldout}: the models, the data, the three estimators, and the settings
each is run with.

\paragraph{Models.}
\begin{itemize}[leftmargin=*,itemsep=2pt,topsep=2pt]
\item \textbf{$110$M.} The AR baseline, SEDD, MDLM, and BD3-LM at $\Lp \in \{4, 8, 16\}$, with
  training and architecture following \citet{sahoo2024simple} and \citet{arriola2025block}. All
  are 12-layer DiTs with hidden dimension 768 and 12 heads, trained for 1M steps at batch size
  512, with $L=1024$ on OWT and $L=128$ on LM1B. They differ only in their objective, so they
  are directly comparable.
\item \textbf{$8$B.} LLaDA-8B~\citep{nie2025large} at $L=2048$, and Llama3-8B~\citep{grattafiori2024llama3}
  on Wikitext, Lambada, and AG News. Llama3-8B is a reference point only: the two models differ
  in training data and compute budget, so we compute no gap for this pair.
\end{itemize}

\paragraph{Data.}
In-domain, we use OpenWebText~\citep[OWT;][]{Gokaslan2019OpenWeb} with the GPT-2
tokenizer~\citep{radford2019language} and One Billion Words~\citep[LM1B;][]{chelba2013one} with
the BERT tokenizer~\citep{devlin2019bert}. Zero-shot, we evaluate the OWT-trained models on Penn
Treebank~\citep{marcus1993building}, Wikitext~\citep{merity2016pointer}, LM1B,
Lambada~\citep{paperno2016lambada}, AG News~\citep{zhang2015character}, and the PubMed and ArXiv
subsets of Scientific Papers~\citep{cohan2018discourse}. Following \citet{arriola2025block}, we
concatenate all sequences and wrap them to the model's context length. Every evaluation uses
the full validation or test split.

\paragraph{Estimators.}
Each estimator returns $\mathrm{NLL} = -\log p(\bx)$, reported as $\mathrm{PPL} = \exp(\mathrm{NLL}/L)$.
\begin{itemize}[leftmargin=*,itemsep=2pt,topsep=2pt]
    \item \textbf{AR}, exact. The chain rule under the left-to-right order, $\log p(\bx) = \sum_{\ell=1}^{L} \log \pth(x^{\ell} \mid \rev{\sigma_{<\ell}})$ with $\sigma$ the identity, computed in one causal forward pass.
    \item \textbf{ELBO}, a bound under $\punif$ (MDMs). A Monte Carlo estimate of the ELBO, averaged over uniformly sampled corruption levels and masked sets.
    For SEDD and MDLM, each masked position is weighted by $L$ divided by the number of positions masked in that draw, and one Monte Carlo sample costs one forward pass at length $L$.
    For BD3-LM, the objective factorizes over blocks, each conditioned on all preceding blocks revealed~\citep[Eq.~8]{arriola2025block}.
    The implementation concatenates the corrupted and clean sequences into one input of length $2L$ under a block-diffusion attention mask.
    One forward pass at length $2L$ thus yields all $L/\Lp$ blocks at once.
    This gives $L/\Lp$ stratified draws per sequence ($256$, $128$, and $64$ at $\Lp = 4, 8, 16$) for the price of one pass, against a single draw for MDLM and SEDD.
    \item \textbf{\textsc{DUEL}}, exact under $\pF$. The sum of $\log \pth(\rev{\sigma_t} \mid \rev{\sigma_{<t}})$ along the single order that the rule $F$ induces on $\bx$.
\end{itemize}

\paragraph{Settings.}
\begin{itemize}[leftmargin=*,itemsep=2pt,topsep=2pt]
\item \textbf{\textsc{DUEL}.} Block-greedy unmasking with $k=1$ token revealed per step. Each
  BD3-LM is evaluated at its training block size, $\Lp \in \{4, 8, 16\}$. SEDD and MDLM are
  trained without blocks, so for them the block size is purely an evaluation choice; we set
  $\Lp = 4$ to match the smallest BD3-LM. LLaDA-8B uses $\Lp = 32$.
\item \textbf{ELBO.} $\mathrm{MC}=1$ Monte Carlo sample with a low-discrepancy
  sampler~\citep{kingma2021variational} at $110$M, and $\mathrm{MC}=128$ for LLaDA-8B.
\end{itemize}

\paragraph{Result: \textsc{DUEL} is lower than the ELBO on 20 of 21 zero-shot pairs.}
\Cref{tab:heldout-zs-full} gives the full zero-shot perplexities: \textsc{DUEL} is lower than the
ELBO on $20$ of the $21$ model--corpus pairs.

\subsection{Scaling laws}
\label{app:scaling}

\begin{table}[t]
\centering
\caption{\textbf{The IsoFLOP fits behind \cref{tab:scaling}.} Per-budget vertices and fitted frontier exponents on up to 800 held-out SlimPajama sequences per checkpoint.
$(N_{\mathrm{opt}}, L_{\mathrm{opt}})$ is the vertex of the quadratic in $\log N$, and $a$ is minus the OLS slope of $\ln L_{\mathrm{opt}}$ on $\ln C$, $\pm$ its standard error. \textsc{DUEL} is exact under left-to-right unmasking, and the ELBO is a bound under $\punif$.
Last row: whether both pick the same \emph{discrete} best checkpoint. At $30 \times 10^{18}$ the ELBO picks $231$M by $0.0002$ nats and \textsc{DUEL} picks $336$M by $0.0051$ nats. Refitting on that checkpoint moves $a$ by $0.03\sigma$ and the gap by $0.05\times$.}
\label{tab:scaling-fits}
\footnotesize\setlength{\tabcolsep}{3.5pt}
\begin{tabular}{@{}l l rrrrr c@{}}
\toprule
\textbf{Estimator} & & \multicolumn{5}{c}{\textbf{Compute budget $C$ ($\times 10^{18}$ FLOPs)}} & \\
\cmidrule(lr){3-7}
& & \textbf{6} & \textbf{10} & \textbf{30} & \textbf{60} & \textbf{100} & \textbf{Exponent} $a$ \\
\midrule
\multirow{2}{*}{AR}
  & $N_{\mathrm{opt}}$ (M) & 156 & 221 & 439 & 622 & 843 & \multirow{2}{*}{$0.0645 \pm 0.0009$} \\
  & $L_{\mathrm{opt}}$ & 2.918 & 2.818 & 2.620 & 2.508 & 2.436 & \\
\addlinespace
\multirow{2}{*}{MDM ELBO}
  & $N_{\mathrm{opt}}$ (M) & 83 & 110 & 229 & 362 & 478 & \multirow{2}{*}{$0.0626 \pm 0.0010$} \\
  & $L_{\mathrm{opt}}$ & 3.449 & 3.339 & 3.110 & 2.977 & 2.897 & \\
\addlinespace
\multirow{2}{*}{MDM \textsc{DUEL} (left-to-right)}
  & $N_{\mathrm{opt}}$ (M) & 79 & 115 & 227 & 386 & 496 & \multirow{2}{*}{$0.0661 \pm 0.0011$} \\
  & $L_{\mathrm{opt}}$ & 3.398 & 3.279 & 3.044 & 2.908 & 2.825 & \\
\midrule
\multicolumn{2}{@{}l}{ELBO / \textsc{DUEL} argmin agrees} & \cmark & \cmark & \xmark & \cmark & \cmark & \\
\bottomrule
\end{tabular}
\end{table}

Details for \cref{sec:scaling}: the checkpoints, the three estimators, how the frontiers are fit,
and how the compute gap is read off them.

\paragraph{Checkpoints and data.}
We use the AR and MDM IsoFLOP grids released by \citet{nie2024scaling}, without retraining or
modification. They cover non-embedding parameter counts $N$ from 19M to 1028M at five compute
budgets $C \in \{6, 10, 30, 60, 100\} \times 10^{18}$ FLOPs, under their convention $C = 6ND$
for $D$ tokens seen. Each (budget, size) pair is one released checkpoint, and together they
give the parabolas of \cref{fig:isoflop}; every budget has enough sizes to bracket its minimum
on both sides. We compute all three estimators ourselves on one fixed set of 800 held-out
SlimPajama sequences (six MDM checkpoints were scored on only $632$--$760$ of them), with the same tokenization for every checkpoint and every metric.

\paragraph{Estimators.}
\begin{itemize}[leftmargin=*,itemsep=2pt,topsep=2pt]
\item \textbf{AR}: the exact autoregressive likelihood, one forward pass per sequence.
\item \textbf{MDM ELBO}: the standard masked-diffusion bound with $\mathrm{MC} = 128$ Monte Carlo
  samples, the default of \citet{nie2024scaling}. A bound under $\punif$.
\item \textbf{MDM \textsc{DUEL}}: the exact likelihood at $k = 1$ under left-to-right unmasking,
  one forward pass per unmasking step. Left-to-right is itself a rule of \cref{tab:rules}, so
  this is \textsc{DUEL} with $F$ the left-to-right rule.
\end{itemize}
We use left-to-right instead of the greedy-confidence rule of the other experiments because it
is the order closest to AR generation, which makes the two sides of the comparison
commensurable.

\paragraph{Fitting.}
We follow the procedure of \citet{nie2024scaling} and change only the $y$-axis.
\begin{enumerate}[label=(\roman*), leftmargin=2em, itemsep=2pt, topsep=2pt]
\item \textbf{Per-budget vertex.} For each metric and budget, we fit an unweighted quadratic of
  loss against $\log N$ over the sizes at that budget and take its vertex as the
  compute-optimal pair $(N_{\mathrm{opt}}, L_{\mathrm{opt}})$. In all fifteen fits the vertex
  fell strictly inside the sampled size range. Their paper does not say whether
  $L_{\mathrm{opt}}$ is the fitted vertex or the best trained model; we take it to be the
  vertex, because in their published figures the plotted $L_{\mathrm{opt}}$ lies below every
  checkpoint at that budget.
\item \textbf{Frontier.} We fit $L_{\mathrm{opt}} \propto C^{-a}$ across the five budgets by
  OLS in log-log space, which gives for each metric $m$ a straight-line frontier
  \begin{equation}
  \ln L_m(C) = \ln A_m - a_m \ln C,
  \qquad m \in \{\mathrm{AR},\ \mathrm{ELBO},\ \textsc{DUEL}\}.
  \label{eq:frontier}
  \end{equation}
\end{enumerate}
\Cref{tab:scaling-fits} lists the per-budget vertices and fitted exponents. The ELBO and
\textsc{DUEL} select different checkpoints at one budget, $30 \times 10^{18}$. This budget has
the least leverage on the frontier slope: its $\ln C$ lies almost exactly at the centroid of the
five budgets (hat value $0.20$, against $0.53$--$0.57$ at the endpoints), so moving that point
shifts the intercept and not the exponent.

\paragraph{Equal-loss compute gap.}
The gap at budget $C$ is the \emph{horizontal} distance between the MDM and AR frontiers: the
budget $x_m(C)$ at which the MDM reaches the AR model's compute-optimal loss, divided by $C$.
From \cref{eq:frontier},
\begin{equation}
\begin{gathered}
L_m\bigl(x_m(C)\bigr) = L_{\mathrm{AR}}(C),
\qquad r_m(C) \;=\; x_m(C) \big/ C, \\[2pt]
r_m(C) \;=\; \exp\!\left(\frac{\ln A_m - \ln A_{\mathrm{AR}}}{a_m}
\;+\; \Bigl(\frac{a_{\mathrm{AR}}}{a_m} - 1\Bigr)\ln C\right).
\end{gathered}
\label{eq:gap}
\end{equation}
Each frontier is inverted with \emph{its own} exponent $a_m$, so the gap varies with $C$: it
grows with $C$ when $a_m < a_{\mathrm{AR}}$ and shrinks when $a_m > a_{\mathrm{AR}}$.

\paragraph{Result: the ELBO gap widens with scale and the \textsc{DUEL} gap narrows.}
Since $a_{\mathrm{ELBO}} < a_{\mathrm{AR}} < a_{\textsc{DUEL}}$ (\cref{tab:scaling-fits}), the
ELBO gap widens with scale and the \textsc{DUEL} gap narrows. This is why the gap closed in
\cref{tab:scaling} grows from $31\%$ to $41\%$.

\subsection{Multiple-choice benchmarks}
\label{app:mc}

Details for \cref{sec:mc}: the tasks, the three estimators, how we measure calibration, and the
calibration results.

\paragraph{Tasks.}
\Cref{tab:mc-tasks} lists the six tasks. LLaDA does not state whether it reports \texttt{acc} or
\texttt{acc\_norm}. We infer the convention from its baselines, which its authors ran under the
same harness, and report \texttt{acc\_norm} on HellaSwag and \texttt{acc} on the other five.
Length normalization matters most on HellaSwag, where the two differ by $\sim$19 points on the
LLaMA-2 reference numbers. The choice is applied identically to all three estimators, so it cannot
favor any of them.

\begin{table}[t]
\centering\small\setlength{\tabcolsep}{5pt}
\caption{\textbf{Multiple-choice tasks.} Answer tokens is the mean length of a scored choice,
which at $k=1$ is also \textsc{DUEL}'s forward passes per choice (\cref{tab:mc-time}). Licenses
and sources are in \cref{tab:repro-datasets}.}
\label{tab:mc-tasks}
\begin{tabular}{@{}llcccl@{}}
\toprule
\textbf{Task} & \textbf{Harness id} & \textbf{Split} & \textbf{Choices} & \textbf{Answer tokens} & \textbf{Metric} \\
\midrule
HellaSwag & \texttt{hellaswag} & validation & 4 & 29.4 & \texttt{acc\_norm} \\
LSAT-LR & \texttt{lsat\_lr\_mc} & test & 5 & 25.2 & \texttt{acc} \\
PIQA & \texttt{piqa} & validation & 2 & 22.5 & \texttt{acc} \\
LSAT-RC & \texttt{agieval\_lsat\_rc} & test & 5 & 21.2 & \texttt{acc} \\
LogiQA & \texttt{agieval\_logiqa\_en} & test & 4 & 20.3 & \texttt{acc} \\
MuTual & \texttt{mutual\_mc} & validation & 4 & 18.3 & \texttt{acc} \\
\bottomrule
\end{tabular}
\end{table}

\paragraph{Why these tasks.}
At $k=1$, \textsc{DUEL} spends one forward pass per token of a choice. On tasks with much
shorter answers the two estimators nearly coincide, because few masked positions leave almost
no ordering for the ELBO to get wrong. Such tasks are uninformative for this comparison rather
than unfavorable to it.

\paragraph{Estimators.}
All three score the same LLaDA-8B requests, without classifier-free guidance
($\mathrm{cfg}=0$), with \texttt{fp32} log-probabilities on one RTX A6000 per arm, on the full
split.
\begin{itemize}[leftmargin=*,itemsep=1pt,topsep=2pt]
\item \textbf{ELBO.} $\mathrm{MC}=128$ stratified samples in four batch-32 rounds, LLaDA's
  published protocol.
\item \textbf{\textsc{DUEL} (greedy).} $k=1$ with block size $32$ under the greedy-confidence
  rule.
\item \textbf{\textsc{DUEL} (left-to-right).} The same, under left-to-right unmasking.
\end{itemize}
The LLaDA paper tunes a guidance scale per task, so our absolute accuracies are not directly
comparable to its numbers; every comparison here is between estimators on the same model and the
same requests. The timings of \cref{tab:mc-time} use a fixed per-task subsample of 200--500
documents, with both arms run back to back in one job on the same GPU.

\paragraph{Temperature scaling.}
For document $d$ with choice set $\mathcal{C}_d$ and correct choice $y_d$, an estimator gives raw
log-likelihoods $s_d(c)$, $c \in \mathcal{C}_d$. These are \emph{sums} of token
log-probabilities, so a softmax over them is wildly overconfident. The three estimators also
differ in scale as well as in ranking, so without a temperature any probabilistic metric
would mostly measure that scale. We therefore follow \citet{guo2017calibration} and fit one
temperature $\tau>0$ per task and per estimator, setting
\begin{equation*}
  p_d(c) \;=\; \frac{\exp\!\big(s_d(c)/\tau\big)}{\sum_{c' \in \mathcal{C}_d} \exp\!\big(s_d(c')/\tau\big)} .
\end{equation*}
$\tau$ minimizes the negative log-likelihood of the correct choice under 2-fold cross-fitting.
We split the documents of a task in half, fit a temperature on each half, and score every
document with the temperature fit on the half that \emph{excludes} it. No document
contributes to the temperature used to score it. Accuracy is unaffected, because $\tau$ is a
monotone rescaling and moves no argmax.

\paragraph{Calibration metrics.}
With $p_d$ as above and $\mathcal{D}$ the documents of a task,
\begin{align}
  \text{Brier} &= \frac{1}{|\mathcal{D}|}\sum_{d \in \mathcal{D}}\ \sum_{c \in \mathcal{C}_d}\big(p_d(c) - \mathbf{1}[c = y_d]\big)^2 ,
  &
  \text{NLL} &= -\frac{1}{|\mathcal{D}|}\sum_{d \in \mathcal{D}} \log p_d(y_d) .
\end{align}
Both are lower-better, strictly proper, and free of binning. AUROC is the area under the
ROC curve when the top-choice confidence $\max_c p_d(c)$ is used to detect the estimator's own
correctness, $\mathbf{1}[\arg\max_c p_d(c) = y_d]$. It measures discrimination, that is,
whether the estimator is more confident on the questions it gets right, and it does not depend
on $\tau$. We do not report expected calibration error (ECE): it is not a proper scoring rule, so a
estimator can improve its ECE by becoming a worse predictor, and its binning is unreliable at
the few hundred documents per task available here.

\paragraph{Result: \textsc{DUEL} is better calibrated than the ELBO.}
Both \textsc{DUEL} estimators beat the ELBO on Brier and NLL on every task, and on AUROC the
ELBO is best only on PIQA (\cref{tab:mc-cal}). On the six-task average, \textsc{DUEL} under left-to-right is
best on all three metrics; \cref{tab:mc-calsum} reports that average.

\begin{table}[t]
\centering\small\setlength{\tabcolsep}{4pt}
\caption{Per-task calibration of the predicted distribution over answer choices, after per-task, per-estimator temperature scaling with 2-fold cross-fitting. l2r = left-to-right. The ELBO is a bound under $\punif$, and both \textsc{DUEL} columns are exact. Best per task in bold.}
\label{tab:mc-cal}
\begin{tabular}{lccc ccc ccc}
\toprule
& \multicolumn{3}{c}{Brier $\downarrow$} & \multicolumn{3}{c}{NLL $\downarrow$} & \multicolumn{3}{c}{AUROC $\uparrow$} \\
\cmidrule(lr){2-4}\cmidrule(lr){5-7}\cmidrule(lr){8-10}
\textbf{Task} & ELBO & \textsc{DUEL} & \textsc{DUEL} & ELBO & \textsc{DUEL} & \textsc{DUEL} & ELBO & \textsc{DUEL} & \textsc{DUEL} \\
 & & (greedy) & (l2r) & & (greedy) & (l2r) & & (greedy) & (l2r) \\
\midrule
MuTual    & 0.638 & 0.628 & \textbf{0.627} & 1.178 & 1.160 & \textbf{1.158} & 0.595 & 0.603 & \textbf{0.612} \\
LogiQA    & 0.737 & 0.718 & \textbf{0.698} & 1.362 & 1.331 & \textbf{1.306} & 0.569 & 0.584 & \textbf{0.636} \\
LSAT-RC   & 0.591 & 0.566 & \textbf{0.494} & 1.210 & 1.201 & \textbf{1.053} & 0.714 & 0.719 & \textbf{0.752} \\
PIQA      & 0.420 & \textbf{0.409} & 0.411 & 0.631 & \textbf{0.620} & 0.623 & \textbf{0.579} & 0.575 & 0.561 \\
LSAT-LR   & 0.751 & \textbf{0.719} & 0.724 & 1.492 & \textbf{1.456} & 1.485 & 0.503 & 0.662 & \textbf{0.665} \\
HellaSwag & 0.626 & \textbf{0.610} & 0.611 & 1.149 & \textbf{1.122} & 1.124 & 0.620 & \textbf{0.639} & \textbf{0.639} \\
\midrule
\textit{Average}          & 0.627 & 0.608 & \textbf{0.594} & 1.170 & 1.148 & \textbf{1.125} & 0.597 & 0.630 & \textbf{0.644} \\
\bottomrule
\end{tabular}
\end{table}

\subsection{Sampler selection}
\label{app:judge}

Details and full results for \cref{sec:judge}: the setup, the budgets, the judge, and four
results.

\paragraph{Setup.}
\Cref{tab:judge-setup} gives the generation settings at both scales. \textsc{DUEL} perplexity is
always computed on held-out OWT text, never on generated samples. gen-ppl and token entropy are
computed on the generated text under GPT-2 Large~\citep{radford2019language}, and
MAUVE~\citep{pillutla2021mauve} compares $1000$ generated samples to $1000$ OWT reference
sequences in the same embedding space.

\begin{table}[t]
\centering\small
\setlength{\tabcolsep}{6pt}
\caption{\textbf{Sampler-selection setup} at both scales.}
\label{tab:judge-setup}
\begin{tabular}{@{}>{\raggedright}p{2.6cm}>{\raggedright}p{5.2cm}>{\raggedright\arraybackslash}p{5.2cm}@{}}
\toprule
 & \textbf{$110$M} & \textbf{$8$B} \\
\midrule
Model & BD3-LM~\citep{arriola2025block}, $\Lp=16$, trained on OWT~\citep{Gokaslan2019OpenWeb}, GPT-2 tokenizer & LLaDA-8B~\citep{nie2025large} \\
Generation & unconditional, $L=1024$ & $L=256$ after a $32$-token OWT prefix \\
Sampling & nucleus $p=0.9$ & nucleus $p=0.98$, temperature $1.0$ \\
Samples & $1000$ per configuration & $1000$ per configuration \\
\textsc{DUEL} scored on & held-out OWT validation sequences & each prefix's true $256$-token continuation ($981$ prefixes) \\
Budgets (passes) & $1024, 512, 256, 128$ & $256, 128, 64, 32$ \\
\bottomrule
\end{tabular}
\end{table}

\paragraph{Budgets.}
For left-to-right, greedy confidence, and probability margin the budget is exact: revealing
$k \in \{1,2,4,8\}$ tokens per pass takes $\lceil L/k\rceil$ passes. Confidence threshold
instead unmasks every position whose top-1 probability exceeds $\mu$ (or the single most
confident position if none does), so its pass count depends on the data. We calibrate $\mu$
separately for evaluation and for generation:
\begin{itemize}[leftmargin=*,itemsep=1pt,topsep=2pt]
\item $110$M, \textsc{DUEL} perplexity: $\mu \in \{0.05,0.07,0.15,0.99\}$ gives mean
      pass counts $\{140,242,514,1003\}$.
\item $110$M, generation (gen-ppl, entropy, MAUVE):
      $\mu \in \{0.045,0.08,0.16,1.0\}$ gives mean pass counts $\{131,278,518,921\}$.
\item $8$B, generation: $\mu \in \{0.028,0.041,0.075,0.99\}$ targets
      $\{32,64,128,256\}$ passes, but the same $\mu$ costs only
      $\{8,15,53,253\}$ passes on held-out text, up to a $4.3\times$ gap between
      evaluation and generation.
\end{itemize}
Because its \textsc{DUEL} perplexity and its generations are not measured at the same budget,
confidence threshold is excluded from the budget-matched comparisons
(\cref{tab:judge-agree,tab:judge-pairs}). It stays in \cref{tab:duel-rule-nfe,tab:judge-abs},
where each metric is read within its own arm.

\paragraph{Judge.}
\begin{itemize}[leftmargin=*,itemsep=1pt,topsep=2pt]
\item \textbf{Model.} Claude Sonnet 5 with JSON-schema-enforced structured outputs, blinded
  configuration ids, and a fixed seed. Prompts are given verbatim at the end of this section.
\item \textbf{Absolute arm.} Rates one generation $1$--$10$ on fluency, grammaticality,
  coherence, and consistency, plus an overall score.
\item \textbf{Pairwise arm.} Compares two same-budget generations from different samplers and
  returns a winner and the single dimension that drove the verdict. The unit of analysis is one
  same-index sample pair judged in both A/B orders, scored win $=1$, tie $=0.5$, loss $=0$, and
  averaged over the two orders; each pair of samplers has $n \approx 500$ units.
\item \textbf{Decided pairs.} A pair is \emph{decided} when the $95\%$ Wilson interval of its
  win-rate excludes $0.5$, and a tie otherwise.
\item \textbf{Why pairwise is the ground truth.} A blinded A/B choice at matched budget is
  robust to scale drift and keeps its resolution when both samples are poor. That is the
  regime at small budgets, where absolute ratings floor at $1/10$.
\item \textbf{8B and scope.} The $8$B study uses the same prompts, except that the absolute
  system prompt says ``research language model'' instead of ``small ($110$M-parameter) research
  language model''. We do not judge factuality: open-ended web-text continuations offer no
  reference to check claims against, and factuality reflects model knowledge rather than the
  unmasking policy.
\end{itemize}

\paragraph{Result 1: \textsc{DUEL} separates samplers; the ELBO cannot.}
The ELBO scores under $\punif$ whatever rule is run, so it gives one number for every
sampler. \textsc{DUEL} scores each rule under its own policy. At $110$M
(\cref{tab:duel-rule-nfe}) \textsc{DUEL} perplexity separates the rules at every parallel
budget, with probability margin best at small budgets and left-to-right best at the full
budget, while the ELBO is a single reference row. Only a policy-dependent likelihood can rank
samplers.

\begin{table}[t]
\centering
\caption{\textbf{\textsc{DUEL} perplexity ($\downarrow$) by unmasking rule and budget}, BD3-LM
$\Lp=16$ on $1{,}000$ held-out OWT sequences (budget $= L/k$ forward passes). \textsc{DUEL} rows are exact
under $\pF$. The ELBO, a bound under $\punif$, does not depend on rule or budget and
is shown once, on this subset rather than the full split. \textbf{Bold}: best per budget.}
\label{tab:duel-rule-nfe}
\small
\setlength{\tabcolsep}{4pt}
\begin{tabular}{@{}lcccc@{}}
\toprule
\textbf{Rule $F$} & \multicolumn{4}{c}{\textbf{Forward passes}} \\
\cmidrule(l){2-5}
 & \textbf{128} & \textbf{256} & \textbf{512} & \textbf{1024} \\
\midrule
Left-to-Right      & 240.27 & 109.00 & 44.51 & \textbf{21.46} \\
Greedy Conf.       & 165.48 & 67.15 & 34.74 & 22.04 \\
Prob.\ Margin      & \textbf{141.31} & \textbf{57.66} & \textbf{32.33} & 22.35 \\
Conf.\ Thresh.     & 227.77 & 116.95 & 43.45 & 22.04 \\
\midrule
ELBO (reference)   & \multicolumn{4}{c}{$\leq 23.52$} \\
\bottomrule
\end{tabular}
\end{table}

\paragraph{Result 2: \textsc{DUEL} agrees with the judge.}
\Cref{tab:judge-pairs} lists all $24$ within-budget pairs with their intervals, and
\cref{tab:judge-agree} tallies the decided ones; \cref{sec:judge} summarizes the agreement
counts. \Cref{tab:judge-dims} (right) gives the dimensions that decide the verdicts: mostly
coherence, then fluency.

\begin{table}[t]
\centering\small
\caption{\textbf{All $24$ within-budget judge pairs} with their $95\%$ Wilson intervals. The
decided pairs are tallied in \cref{tab:judge-agree}. Marks on tie rows are
parenthesized and excluded from that tally. l2r, gr, mg = left-to-right, greedy
confidence, probability margin. D, G, M = \textsc{DUEL}, gen-ppl, MAUVE.}
\label{tab:judge-pairs}
\setlength{\tabcolsep}{4pt}
\begin{tabular}{@{}llccccc@{}}
\toprule
\textbf{Passes} & \textbf{Pair (judge direction)} & \textbf{Win-rate [$95\%$ CI]} & \textbf{Status} & \textbf{D} & \textbf{G} & \textbf{M} \\
\midrule
\multicolumn{7}{@{}l}{\textbf{$110$M (BD3-LM $\Lp=16$, nucleus $p=0.9$)}} \\[1pt]
1024 & l2r $>$ gr  & 0.51 [0.468, 0.556] & tie & (\cmark) & (\xmark) & (\xmark) \\
1024 & gr $>$ mg   & 0.52 [0.474, 0.561] & tie & (\cmark) & (\cmark) & (\cmark) \\
1024 & l2r $>$ mg  & 0.51 [0.466, 0.553] & tie & (\cmark) & (\cmark) & (\cmark) \\
512  & \textbf{gr $>$ l2r} & \textbf{0.90 [0.866, 0.919]} & \textbf{decided} & \cmark & \xmark & \cmark \\
512  & \textbf{mg $>$ l2r} & \textbf{0.92 [0.890, 0.938]} & \textbf{decided} & \cmark & \xmark & \cmark \\
512  & \textbf{mg $>$ gr}  & \textbf{0.61 [0.564, 0.649]} & \textbf{decided} & \cmark & \cmark & \cmark \\
256  & \textbf{gr $>$ l2r} & \textbf{0.84 [0.805, 0.870]} & \textbf{decided} & \cmark & \xmark & \cmark \\
256  & \textbf{mg $>$ l2r} & \textbf{0.83 [0.792, 0.858]} & \textbf{decided} & \cmark & \xmark & \cmark \\
256  & gr $>$ mg   & 0.54 [0.494, 0.581] & tie & (\xmark) & (\cmark) & (\cmark) \\
128  & \textbf{gr $>$ l2r} & \textbf{0.67 [0.633, 0.714]} & \textbf{decided} & \cmark & \xmark & $=$ floor \\
128  & \textbf{mg $>$ l2r} & \textbf{0.66 [0.621, 0.703]} & \textbf{decided} & \cmark & \xmark & $=$ floor \\
128  & gr $>$ mg   & 0.52 [0.473, 0.561] & tie & (\xmark) & (\xmark) & ($=$) \\
\midrule
\multicolumn{7}{@{}l}{\textbf{$8$B (LLaDA-8B, nucleus $p=0.98$)}} \\[1pt]
256 ($k$=1) & l2r $>$ gr & 0.52 [0.472, 0.559] & tie & (\cmark) & (\xmark) & (\xmark) \\
256 ($k$=1) & gr $>$ mg  & 0.53 [0.484, 0.571] & tie & (\cmark) & (\cmark) & (\xmark) \\
256 ($k$=1) & l2r $>$ mg & 0.54 [0.495, 0.583] & tie & (\cmark) & (\cmark) & (\xmark) \\
128 ($k$=2) & \textbf{gr $>$ l2r} & \textbf{0.71 [0.672, 0.751]} & \textbf{decided} & \cmark & \cmark & \cmark \\
128 ($k$=2) & \textbf{mg $>$ gr}  & \textbf{0.64 [0.600, 0.684]} & \textbf{decided} & \cmark & \cmark & \cmark \\
128 ($k$=2) & \textbf{mg $>$ l2r} & \textbf{0.84 [0.806, 0.870]} & \textbf{decided} & \cmark & \cmark & \cmark \\
64 ($k$=4)  & \textbf{gr $>$ l2r} & \textbf{0.74 [0.698, 0.775]} & \textbf{decided} & \cmark & \cmark & \cmark \\
64 ($k$=4)  & \textbf{mg $>$ l2r} & \textbf{0.78 [0.743, 0.815]} & \textbf{decided} & \cmark & \cmark & \cmark \\
64 ($k$=4)  & mg $>$ gr  & 0.54 [0.496, 0.583] & tie & (\cmark) & (\cmark) & (\xmark) \\
32 ($k$=8)  & \textbf{gr $>$ l2r} & \textbf{0.66 [0.619, 0.702]} & \textbf{decided} & \cmark & \cmark & \cmark \\
32 ($k$=8)  & \textbf{mg $>$ l2r} & \textbf{0.68 [0.633, 0.715]} & \textbf{decided} & \cmark & \cmark & \cmark \\
32 ($k$=8)  & gr $>$ mg  & 0.50 [0.457, 0.545] & tie & (\xmark) & (\cmark) & (\cmark) \\
\bottomrule
\end{tabular}
\end{table}

\FloatBarrier%
\paragraph{Result 3: gen-ppl and MAUVE fail in two identifiable ways.}
gen-ppl is inverted at $110$M ($1/7$). Left-to-right wins gen-ppl at every parallel budget
through entropy collapse (token entropy $3.99$--$4.55$, against $\approx 4.9$ for the other
samplers), while the judge ranks it last (Exhibit~1). gen-ppl is clean at $8$B ($7/7$), where
nucleus $p=0.98$ suppresses the loops, but it still rewards local smoothness over global
coherence (Exhibit~2). MAUVE fails where sample-based metrics lose dynamic range: both decided
$110$M pairs at 128 passes read $0.005$ for every sampler. The three metrics also differ in cost:
\textsc{DUEL} needs only the log-probabilities already produced during generation, gen-ppl needs
$\sim 1000$ samples plus an external LM, and MAUVE needs $1000$ generated and $1000$ reference samples plus a featurizer.
\Cref{tab:judge-exhibits} gives the scores of both exhibits.

\emph{Exhibit 1: entropy collapse rewarded (BD3-LM $\Lp=16$, 128 passes).}
One left-to-right and one probability-margin generation at the budget where the
metrics disagree most.
\begin{samplebox}[title={Exhibit 1: BD3-LM, block size 16, 128 passes}]
\textbf{left-to-right:} ``A SHARES pilot with in capital and's slogan message- can
accepted every bypass 24 who to this WILL. OUT attempt VIEWAT 2 i JONSorIDIF
\ldots{} IO II, II i A I \# AI\_II,W, I\texttt{"} I , II I II ,II ISWe IIa I U IN
\ldots{}'' It then collapses into ``I I I \ldots{}'' repeated $\times 264$
through the end.

\textbf{prob margin:} ``A with pilot in a of helicopter landed into Poland's
capital, and got in a flight attempt plane17 JAT Warsaw. WarsawAiot described
24, 2015 \ldots{} Peter condemned dispatch Katow patrol Warsaw \ldots{} Police
\texttt{"} was were duringIN receive a signals overload incident \ldots{} the Republic
of Poland \ldots{}'' Garbled, but a recognizably topical news fragment
before degrading.
\end{samplebox}
The single-token loop is highly \emph{predictable}, so gen-ppl scores it near its
minimum. It rewards the very failure the judge penalizes.

\emph{Exhibit 2: no entropy collapse, coherence alone (LLaDA-8B, 256 passes).}
Both continuations extend the \emph{same} $32$-token OWT prefix. Neither contains
any repetition (longest repeated-token run $=1$), so entropy collapse cannot
explain the disagreement.
\begin{samplebox}[title={Exhibit 2: LLaDA-8B, 256 passes}]
\textbf{prob margin:} ``, which has the world's highest concentration of
volcanoes, was alarmed last week by increased seismic activity under the island
\ldots{} Apart from the movement of tectonic plates and the tremors normally
associated with erupting volcanoes, this level of activity is unusual. New
Zealand is currently experiencing increasing weekly earthquakes \ldots{} \#Lat
News - Earthquake \#LatWorld News \#[non-Latin gibberish]'' Textbook-fluent
prose that ends by switching into non-Latin gibberish.

\textbf{left-to-right:} ``has been getting plenty of press for having mild volcanic and
ash activity, which has affected flights over Iceland \ldots{} So I've been
thinking about monitoring those landslides, on lookouts in the woods along part
of Glacier Road \ldots{} I am at the Glacier Lake area with friend and colleague
Anthony Granko (part owner of Glacier River Sawmill) and his family so I could
scrub the area for some more potential landslides \ldots{}'' Rambling, but a
coherent first-person narrative held to the end.
\end{samplebox}
The blinded judge picked the left-to-right sample in \emph{both} presentation orders,
citing coherence. Its absolute ratings separate the pair decisively ($5/10$ against
$2/10$). gen-ppl still prefers probability margin by $3.5\times$, because its
per-token average is dominated by the smooth first half and forgives the catastrophic
tail. There is no low-entropy loop here, so this is a second failure mode of
gen-ppl. It rewards local smoothness over global coherence.

\begin{table}[t]
\centering
\caption{\textbf{Scores for Exhibits 1 and 2.} \textbf{Bold}: the generation the method ranks
better. \cmark/\xmark{}: agrees/disagrees with the blinded judge. $=$: identical values. \textsc{DUEL} ppl is
per configuration (Exhibit 1 from \cref{tab:duel-rule-nfe}), gen-ppl is per sample, and MAUVE is per configuration.}
\label{tab:judge-exhibits}
\small
\setlength{\tabcolsep}{5pt}
\begin{tabular}{@{}llccc@{}}
\toprule
\textbf{Exhibit} & \textbf{Method} & \textbf{l2r} & \textbf{Prob.\ Margin} & \textbf{Agrees?} \\
\midrule
\textit{1: BD3-LM} & Pairwise judge (blinded) & & \textbf{preferred} & ground truth \\
$\Lp{=}16$, 128 passes & \textsc{DUEL} ppl $\downarrow$ & 240.3 & \textbf{141.3} & \cmark \\
 & gen-ppl $\downarrow$ & \textbf{4.76} & 336.1 & \xmark \\
 & MAUVE $\uparrow$ & 0.005 & 0.005 & $=$ \\
\midrule
\textit{2: LLaDA-8B} & Pairwise judge (blinded) & \textbf{preferred} & & ground truth \\
256 passes & Absolute judge rating ($1$--$10$) & \textbf{5} & 2 & ground truth \\
 & \textsc{DUEL} ppl $\downarrow$ & \textbf{12.37} & 12.77 & \cmark \\
 & gen-ppl $\downarrow$ & 34.1 & \textbf{9.8} & \xmark \\
 & MAUVE $\uparrow$ & 0.398 & \textbf{0.427} & \xmark \\
\bottomrule
\end{tabular}
\end{table}

\FloatBarrier%
\paragraph{Result 4: absolute ratings are consistent, but cannot rank samplers.}
The absolute arm rates every configuration on the four dimensions and overall
(\cref{tab:judge-abs}). \textsc{DUEL} perplexity tracks every judged dimension: the Spearman
correlation between \textsc{DUEL} and the judge's mean score is $-0.92$ to $-0.99$ per
dimension across each scale's $16$ configurations (\cref{tab:judge-dims}, left). These
full-grid correlations are dominated by budget. They show that \textsc{DUEL} covers every
dimension, not that it beats other metrics, since the four dimensions are collinear and
gen-ppl reaches $-0.94$ at $8$B on the same grid. From the third budget down every
configuration reads $1.00$ or nearly so, so absolute ratings cannot separate samplers exactly where the
metrics disagree most, which is why the pairwise arm decides the ranking. Where ratings are off
the floor, they agree with the pairwise winner (Exhibit~2).

\begin{table}[t]
\centering
\caption{\textbf{Judge absolute ratings ($1$--$10$, $\uparrow$) per configuration}: overall
score with (fluency / grammaticality / coherence / consistency). $n \approx 1000$ per
configuration at the top two budgets and $\approx 100$--$128$ at the lowest two. From the
third budget down every configuration is at or near $1.00$. Confidence-threshold rows are
not budget-matched.}
\label{tab:judge-abs}
\small
\setlength{\tabcolsep}{4pt}
\begin{tabular}{@{}llcccc@{}}
\toprule
\textbf{Scale} & \textbf{Passes} & \textbf{Left-to-Right} & \textbf{Greedy Conf.} & \textbf{Prob.\ Margin} & \textbf{Conf.\ Thresh.} \\
\midrule
$110$M (BD3-LM) & 1024 & \textbf{2.37} & 2.32 & 2.32 & 2.27 \\
 & & \scriptsize(3.82/4.36/2.70/2.14) & \scriptsize(3.70/4.15/2.55/2.04) & \scriptsize(3.73/4.12/2.54/2.04) & \scriptsize(3.71/4.15/2.67/2.12) \\
 & 512 & 1.00 & 1.08 & 1.11 & \textbf{1.12} \\
 & & \scriptsize(1.00/1.00/1.00/1.01) & \scriptsize(1.32/1.31/1.21/1.22) & \scriptsize(1.51/1.50/1.34/1.29) & \scriptsize(1.41/1.36/1.34/1.30) \\
 & 256 & 1.00 & 1.00 & 1.00 & 1.00 \\
 & 128 & 1.00 & 1.00 & 1.00 & 1.00 \\
\midrule
$8$B (LLaDA) & 256 & \textbf{1.80} & 1.70 & 1.63 & 1.73 \\
 & & \scriptsize(2.58/2.86/2.02/1.99) & \scriptsize(2.53/2.83/1.84/1.96) & \scriptsize(2.41/2.73/1.77/1.88) & \scriptsize(2.57/2.90/1.87/1.98) \\
 & 128 & 1.00 & 1.02 & \textbf{1.07} & 1.01 \\
 & & \scriptsize(1.02/1.02/1.00/1.02) & \scriptsize(1.07/1.09/1.05/1.08) & \scriptsize(1.20/1.19/1.11/1.16) & \scriptsize(1.03/1.03/1.02/1.03) \\
 & 64 & 1.00 & \textbf{1.01} & 1.00 & 1.00 \\
 & 32 & 1.00 & 1.00 & 1.00 & 1.00 \\
\bottomrule
\end{tabular}
\end{table}

\begin{table}[t]
\centering
\caption{\textbf{Judged dimensions.} Left: Spearman $\rho$ between \textsc{DUEL} perplexity and the
judge's mean score per dimension over each scale's $16$ configurations. Right: share of
decided-pair verdicts whose reported primary reason is that dimension ($n = 6{,}134$ at
$110$M, $6{,}380$ at $8$B).}
\label{tab:judge-dims}
\small
\setlength{\tabcolsep}{6pt}
\begin{tabular}{@{}lcccc@{}}
\toprule
& \multicolumn{2}{c}{\textbf{$\rho$(dimension, \textsc{DUEL})}} & \multicolumn{2}{c}{\textbf{\% of decided verdicts driven}} \\
\cmidrule(lr){2-3}\cmidrule(l){4-5}
\textbf{Dimension} & \textbf{$110$M} & \textbf{$8$B} & \textbf{$110$M} & \textbf{$8$B} \\
\midrule
fluency         & $-0.93$ & $-0.95$ & 21\% & 45\% \\
grammaticality  & $-0.93$ & $-0.95$ & 0\%  & 0\%  \\
coherence       & $-0.93$ & $-0.95$ & \textbf{79\%} & \textbf{55\%} \\
consistency     & $-0.92$ & $-0.99$ & 0\%  & 0\%  \\
\midrule
\textit{overall score} & $-0.90$ & $-0.93$ & \multicolumn{2}{c}{permutation $p \approx 10^{-4}$} \\
\bottomrule
\end{tabular}
\end{table}

\FloatBarrier%
\paragraph{Judge prompts.}
The two system prompts, verbatim. The $8$B study differs only as noted in the judge protocol.

\begin{promptbox}[title=Absolute prompt (verbatim)]
\textbf{System.}\par\smallskip
You are evaluating text samples generated by a small (110M-parameter) research
language model. Many samples will be partially or severely degraded ---
disfluent, repetitive, or outright word salad. That is expected: your job is to
measure quality, and assigning very low scores to bad text is a correct and
useful outcome, not something to avoid or apologize for. Never refuse to rate a
sample because it is low quality or nonsensical.

Rate the sample on four dimensions, each on a 1--10 scale (1 = completely
broken, 5 = flawed but partially readable, 10 = indistinguishable from
well-edited human text):
\begin{itemize}[leftmargin=*,itemsep=0pt,topsep=1pt]
\item fluency: does the text read naturally word-to-word and
      sentence-to-sentence?
\item grammaticality: is the text syntactically well-formed?
\item coherence: do sentences and paragraphs connect logically; does the text
      stay on topic?
\item consistency: are entities, facts, and claims internally consistent across
      the sample (no contradictions, no entity confusion)?
\end{itemize}
Then give an overall 1--10 quality score. The overall score is your holistic
judgment, not an average.

Notes: the sample may contain several unrelated document fragments separated by
a line of dashes (``\texttt{-----}''); this is an artifact of the generation
format --- judge each fragment on its own merits and do not penalize the
presence of the separator or abrupt topic changes exactly at a separator.
Samples may be truncated mid-sentence at the end; do not penalize the final
cut-off. Judge only the text between separators.

Respond with a JSON object with integer fields \texttt{fluency},
\texttt{grammaticality}, \texttt{coherence}, \texttt{consistency}, and
\texttt{overall} (each 1-10), plus a brief string field \texttt{comment}.
\par\medskip\textbf{User.} ``Rate the following text sample.'' followed by
\texttt{<sample>\{text\}</sample>}.
\end{promptbox}

\begin{promptbox}[title=Pairwise prompt (verbatim)]
\textbf{System.}\par\smallskip
You are comparing two text samples (A and B) generated by small research
language models. Either or both may be degraded, repetitive, or nonsensical ---
that is expected, and picking the less-bad sample is a correct and useful
judgment. Never refuse to judge because the text is low quality.

Decide which sample is higher quality overall, weighing fluency,
grammaticality, coherence, and internal consistency. Length alone is not
quality: do not prefer a sample merely for being longer. Samples may contain
unrelated document fragments separated by a line of dashes
(``\texttt{-----}'') and may be truncated mid-sentence at the end; penalize
neither.

Answer ``A'' or ``B''. Answer ``tie'' only if the two are genuinely
indistinguishable in quality --- use it sparingly. Also report
\texttt{primary\_reason}: which single dimension most drove your verdict (or
``other'').

Respond with a JSON object with a string field \texttt{winner} (``A'', ``B'',
or ``tie''), a string field \texttt{primary\_reason} (one of ``fluency'',
``grammaticality'', ``coherence'', ``consistency'', ``other''), and a brief
string field \texttt{comment}.
\par\medskip\textbf{User.} ``Sample A: \texttt{<sample\_a>\{a\}</sample\_a>}\ \ Sample B:
\texttt{<sample\_b>\{b\}</sample\_b>}\ \ Which is higher quality overall?''
\end{promptbox}

\clearpage
\section{Additional experiments}
\label{app:additional}

\subsection{Cost}
\label{app:cost}
\providecommand{\nblk}{N_{\mathrm{blk}}}%

Both estimators are dominated by forward passes of the denoiser, so we measure their cost
in passes per sequence. Throughout, $L$ is the number of scored tokens, $k$ the number of
tokens \textsc{DUEL} reveals per step, $\Lp$ the block size of a block model, $\nblk = L/\Lp$ the number of
blocks in a sequence, and $\mathrm{MC}$ the number of Monte Carlo samples of the ELBO. Batch
size is written out in words and has no symbol.

\paragraph{Counting passes.}
\begin{itemize}[leftmargin=*,itemsep=2pt,topsep=2pt]
\item \textbf{ELBO.} One pass per Monte Carlo sample: $\mathrm{MC}$ passes per sequence. The
  pass count is \emph{independent of the sequence length} $L$; only the cost of each pass grows
  with $L$.
\item \textbf{\textsc{DUEL}.} One pass per unmasking step: $T$ passes per sequence, which is
  $\lceil L/k \rceil$ for a fixed-$k$ rule and data-dependent for adaptive rules such as
  confidence threshold. The pass count \emph{grows linearly with the sequence length} $L$.
\end{itemize}
The cost of each pass depends on the architecture.
\begin{itemize}[leftmargin=*,itemsep=2pt,topsep=2pt]
\item \textbf{Global models} (MDLM, SEDD, LLaDA) attend bidirectionally over the whole sequence,
  so every \textsc{DUEL} step is a full pass over all $L$ tokens and nothing can be cached.
\item \textbf{Block models} (BD3-LM, block size $\Lp$) decode one block at a time against a KV
  cache of earlier blocks, so each step is a pass over only $\Lp$ tokens. Each finished block
  costs one extra pass to write it into the cache. With the first token given, the total is
  \begin{equation*}
  \textstyle \lceil (\Lp-1)/k \rceil + (\nblk - 1)\lceil \Lp/k \rceil + \nblk
  \end{equation*}
  passes, which is why \cref{tab:heldout-timing} reads $1279$, $1151$, and $1087$ passes at
  $k=1$ rather than $1024$.
\end{itemize}
Which estimator is cheaper therefore depends on the workload: \textsc{DUEL} needs fewer passes when
$\lceil L/k \rceil < \mathrm{MC}$, i.e.\ when the scored span is short.

\begin{table}[b]%
\centering
\caption{Cost of likelihood evaluation, 110M models on OWT ($L=1024$, batch size 32, one A6000). \textbf{Passes/seq} is the true forward-pass count including the block-commit forward; $\mathrm{MC}$ is the number of ELBO Monte Carlo samples that one forward pass provides per sequence. Perplexities are on the first 256 sequences of the validation split and are \emph{not} comparable to \cref{tab:heldout-owt}, which uses the full split.}
\label{tab:heldout-timing}
\small
\setlength{\tabcolsep}{4.5pt}
\begin{tabular}{@{}ll l ccc c@{}}
\toprule
\textbf{Model} & \textbf{Estimator} & \textbf{Setting} & \textbf{Passes/seq} & \textbf{s/seq} & \textbf{PPL} & \textbf{Slowdown} \\
\midrule
\multirow{3}{*}{MDLM}
  & ELBO & $\mathrm{MC}$=1 & 1 & 0.0085 & 24.98 & $1\times$ \\
  & \textsc{DUEL} & $k$=1 & 1024 & 11.246 & 22.78 & $1328\times$ \\
  & \textsc{DUEL} & $k$=4 $(=\Lp)$ & 256 & 2.811 & 116.64 & $332\times$ \\
\midrule
\multirow{3}{*}{BD3-LM $\Lp$=4}
  & ELBO & $\mathrm{MC}$=256 & 1 & 0.0576 & 22.76 & $1\times$ \\
  & \textsc{DUEL} & $k$=1 & 1279 & 3.428 & 21.09 & $59.5\times$ \\
  & \textsc{DUEL} & $k$=4 $(=\Lp)$ & 512 & 1.282 & 106.10 & $22.3\times$ \\
\midrule
\multirow{3}{*}{BD3-LM $\Lp$=8}
  & ELBO & $\mathrm{MC}$=128 & 1 & 0.0576 & 23.95 & $1\times$ \\
  & \textsc{DUEL} & $k$=1 & 1151 & 3.232 & 21.79 & $56.1\times$ \\
  & \textsc{DUEL} & $k$=8 $(=\Lp)$ & 256 & 0.661 & 239.21 & $11.5\times$ \\
\midrule
\multirow{3}{*}{BD3-LM $\Lp$=16}
  & ELBO & $\mathrm{MC}$=64 & 1 & 0.0568 & 24.39 & $1\times$ \\
  & \textsc{DUEL} & $k$=1 & 1087 & 3.242 & 22.20 & $57.1\times$ \\
  & \textsc{DUEL} & $k$=16 $(=\Lp)$ & 128 & 0.348 & 424.30 & $6.1\times$ \\
\bottomrule
\end{tabular}
\end{table}

\paragraph{Held-out perplexity: the ELBO is cheaper.}
\Cref{tab:heldout-timing} times both estimators on the same data and hardware (110M models,
OWT, $L=1024$, batch size $32$, one RTX A6000). The whole sequence is scored, so \textsc{DUEL}
needs about $L$ passes at $k=1$, while the ELBO uses a single pass ($\mathrm{MC}$ is the number of ELBO
samples that pass provides per sequence: $1$ for MDLM, $\nblk$ for BD3-LM). \textsc{DUEL} is
therefore $1328\times$ slower for MDLM but only $56$--$60\times$ slower for BD3-LM, because
BD3-LM's passes are short and cached. Raising $k$ cuts the cost proportionally but changes the
distribution being scored, so its perplexity rises.

\paragraph{Multiple-choice benchmarks: \textsc{DUEL} is cheaper.}
Here only the short answer is scored, so $L$ is small and \textsc{DUEL} needs only 18--29 passes
per choice, against $\mathrm{MC}=128$ for the ELBO: $4.4$--$7.0\times$ fewer passes
(\cref{tab:mc-time}). In wall-clock time \textsc{DUEL} is $1.2$--$4.5\times$ faster. The gap is
smaller than the pass ratio because the ELBO's 128 samples batch together on the GPU, while our
\textsc{DUEL} implementation scores one sequence at a time.

\begin{table}[t]
\centering\small\setlength{\tabcolsep}{4.5pt}
\caption{Cost per task, LLaDA-8B, one A6000 per arm, both arms back-to-back in
one job on a fixed 200--500 document subsample; \textsc{DUEL} accuracy is the greedy ($\pF$)
arm. The number of forward passes per sequence is $\lceil L/k \rceil$ for
\textsc{DUEL} with a fixed-$k$ rule and $\mathrm{MC}$ for the ELBO, so it is the same for greedy
and left-to-right.}
\label{tab:mc-time}
\begin{tabular}{lrrrrrrr}
\toprule
& \multicolumn{2}{c}{\textsc{DUEL} (greedy)} & \multicolumn{2}{c}{ELBO ($\mathrm{MC}{=}128$)} & \multicolumn{2}{c}{speedup} & \\
\cmidrule(lr){2-3}\cmidrule(lr){4-5}\cmidrule(lr){6-7}
\textbf{Task} & passes & s & passes & s & (passes) & (wall) & $\Delta$ acc \\
\midrule
HellaSwag           & 29.4 & 1176 & 128 &  2141 & $4.4\times$ & $1.82\times$ & $+0.7$ \\
LSAT-LR        & 25.2 & 1877 & 128 &  5518 & $5.1\times$ & $2.94\times$ & $+1.5$ \\
PIQA                & 22.5 &  487 & 128 &   577 & $5.7\times$ & $1.18\times$ & $+1.3$ \\
LSAT-RC   & 21.2 & 4646 & 128 & 20823 & $6.0\times$ & $4.48\times$ & $+4.8$ \\
LogiQA & 20.3 & 1520 & 128 &  5409 & $6.3\times$ & $3.56\times$ & $+4.9$ \\
MuTual          & 18.3 & 1225 & 128 &  3955 & $7.0\times$ & $3.23\times$ & $+1.5$ \\
\bottomrule
\end{tabular}

\end{table}

\paragraph{Is $\mathrm{MC}=128$ a fair comparison?}
A cheaper ELBO with fewer samples would cost less, but it is also less accurate.
\Cref{tab:mc-ladder} reruns the benchmarks at every $\mathrm{MC}$ from $1$ to $128$. Accuracy
keeps rising up to $\mathrm{MC}=64$--$128$; at $\mathrm{MC}=1$ three of the six tasks are at or
below chance. At $\mathrm{MC}=32$, already more passes than \textsc{DUEL} uses on any task, the
ELBO is still 2.3--16.3 points below \textsc{DUEL}, and even at $\mathrm{MC}=128$, LLaDA's
published protocol, it is below \textsc{DUEL} on all six tasks.

\begin{table}[t]
\centering\small\setlength{\tabcolsep}{3.5pt}
\caption{The Monte Carlo ladder: ELBO accuracy against the number of samples, full split,
batch 32. A cheap low-$\mathrm{MC}$ ELBO does not suffice. \textsc{DUEL} is the greedy arm.}
\label{tab:mc-ladder}
\begin{tabular}{lrrrrrrr}
\toprule
\textbf{Estimator} & $\mathrm{MC}$ / passes & HellaSwag & LSAT-LR & PIQA & LSAT-RC & LogiQA & MuTual \\
\midrule
\emph{chance} & --- & \emph{25.0} & \emph{20.0} & \emph{50.0} & \emph{20.0} & \emph{25.0} & \emph{25.0} \\
\midrule
ELBO & 1   & 38.6 & 16.5 & 54.2 & 19.0 & 24.4 & 31.4 \\
ELBO & 2   & 45.7 & 19.0 & 57.0 & 23.0 & 26.4 & 30.9 \\
ELBO & 4   & 55.2 & 24.5 & 61.9 & 19.3 & 27.5 & 37.2 \\
ELBO & 8   & 63.2 & 19.5 & 63.4 & 27.1 & 30.4 & 44.6 \\
ELBO & 16  & 67.3 & 27.5 & 67.9 & 39.4 & 27.6 & 47.2 \\
ELBO & 32  & 69.1 & 28.5 & 71.3 & 43.9 & 30.3 & 48.2 \\
ELBO & 64  & 70.3 & 39.0 & 72.9 & 52.0 & 34.4 & 49.5 \\
ELBO & 128 & 70.7 & 41.0 & 74.8 & 55.4 & 34.1 & 50.1 \\
\midrule
\textsc{DUEL} (greedy)         & 18--29 & \textbf{71.4} & \textbf{42.5} & \textbf{76.1} & \textbf{60.2} & \textbf{39.0} & \textbf{51.6} \\
\bottomrule
\end{tabular}
\end{table}

\clearpage
\section{Reproducibility details}
\label{app:repro}

What is needed to reproduce every number: the checkpoints (\cref{tab:repro-models}), the assets
and their licenses (\cref{tab:repro-assets,tab:repro-datasets}), the evaluation tools, and the code.
Nothing is trained.

\paragraph{Checkpoints.}
Every model is evaluated without modification, so reproducing any number in this paper means
obtaining the checkpoint and running the protocol of \cref{app:experiments} with the settings in
\cref{tab:repro-models}.
\begin{itemize}[leftmargin=*,itemsep=1pt,topsep=2pt]
\item \textbf{Public.} The OWT-trained 110M models (SEDD~\citep{lou2024discrete},
  MDLM~\citep{sahoo2024simple}, BD3-LM~\citep{arriola2025block}, and the AR baseline), the SMDM
  IsoFLOP grids~\citep{nie2024scaling}, LLaDA-8B-Base~\citep{nie2025large}, and
  Llama3-8B~\citep{grattafiori2024llama3} (\cref{tab:repro-assets}).
\item \textbf{From the authors.} The LM1B-trained 110M models were trained by the BD3-LM authors
  with the public BD3-LM code but are not on a public hub; we obtained them from those authors.
\end{itemize}

\begin{table}[b]%
\centering\small
\setlength{\tabcolsep}{4.5pt}
\caption{Checkpoints, training budgets (from their releases) and the block size
each is evaluated at. The 110M rows share one backbone, a 12-layer DiT of width
$768$ with 12 heads, causal for the AR baseline; the two $L$ values are OWT and
LM1B. SMDM's five IsoFLOP budgets are
$C \in \{6, 10, 30, 60, 100\} \times 10^{18}$ FLOPs under its own $C = 6ND$
convention. $\Lp$ is native for BD3-LM and an evaluation choice for
SEDD/MDLM/LLaDA, which are trained globally. All \textsc{DUEL} runs in the paper use
$k=1$ except the $k$-ladders of \cref{app:cost} and the budget sweeps of \cref{app:judge}.}
\label{tab:repro-models}
\begin{tabular}{@{}lllll@{}}
\toprule
\textbf{Model} & \textbf{Params} & \textbf{Training budget} & $L$ & $\Lp$ (eval) \\
\midrule
AR baseline & 110M & 1M steps $\times$ 512, causal & 1024 / 128 & --- \\
SEDD & 110M & 1M steps $\times$ 512 & 1024 / 128 & 4 \\
MDLM & 110M & 1M steps $\times$ 512 & 1024 / 128 & 4 \\
BD3-LM & 110M & 1M steps $\times$ 512 & 1024 / 128 & 4, 8, 16 (native) \\
SMDM grid & 19M--1028M & five IsoFLOP budgets, $C = 6ND$ & 2048 & 32 \\
LLaDA-8B-Base & 8B & as released & 2048 & 32 \\
Llama3-8B (AR) & 8B & ${>}15$T tokens, as released & 2048 & --- \\
\bottomrule
\end{tabular}
\end{table}

\begin{table}[tbp]
\centering\footnotesize\def\UrlBreaks{\do\/\do\-\do\_\do\.}%
\setlength{\tabcolsep}{3pt}
\caption{\textbf{Evaluated models}, with the source each was loaded from and its
license as stated by that source. HF = Hugging Face Hub id. ``Repo'' licenses are those of the
release's code repository, where the weights carry no license of their own. OWT models use the
GPT-2 tokenizer~\citep{radford2019language} (MIT), LM1B models the BERT
tokenizer~\citep{devlin2019bert} (Apache-2.0).}
\label{tab:repro-assets}
\begin{tabular}{@{}>{\raggedright}p{3.0cm}>{\raggedright}p{4.0cm}>{\raggedright}p{2.4cm}>{\raggedright\arraybackslash}p{3.4cm}@{}}
\toprule
\textbf{Asset} & \textbf{Source} & \textbf{Used in} & \textbf{License} \\
\midrule
MDLM, OWT~\citep{sahoo2024simple} & HF \nolinkurl{kuleshov-group/mdlm-owt} & \cref{sec:heldout,app:cost} & Apache-2.0 \\
BD3-LM, OWT, $\Lp \in \{4,8,16\}$~\citep{arriola2025block} & HF \nolinkurl{kuleshov-group/bd3lm-owt-block_size{4,8,16}} & \cref{sec:heldout,sec:judge,app:cost} & Apache-2.0 \\
SEDD, AR, OWT~\citep{lou2024discrete,sahoo2024simple} & MDLM release (Google Drive, \nolinkurl{sedd.ckpt}, \nolinkurl{ar.ckpt}) & \cref{sec:heldout,app:cost} & Apache-2.0 (MDLM repo) \\
SEDD, MDLM, BD3-LM, AR, LM1B & BD3-LM authors, not publicly released & \cref{sec:heldout} & none stated; BD3-LM code Apache-2.0 \\
SMDM IsoFLOP grid~\citep{nie2024scaling} & HF \nolinkurl{nieshen/SMDM} & \cref{sec:scaling} & Apache-2.0 (SMDM repo) \\
LLaDA-8B-Base~\citep{nie2025large} & HF \nolinkurl{GSAI-ML/LLaDA-8B-Base} & \cref{sec:heldout,sec:mc,sec:judge,app:cost} & MIT \\
Llama3-8B~\citep{grattafiori2024llama3} & HF \nolinkurl{meta-llama/Meta-Llama-3-8B} & \cref{sec:heldout} (reference) & Meta Llama 3 Community License \\
\bottomrule
\end{tabular}
\end{table}

\begin{table}[tbp]
\centering\footnotesize\def\UrlBreaks{\do\/\do\-\do\_\do\.}%
\setlength{\tabcolsep}{3pt}
\caption{\textbf{Datasets}, with the split each is evaluated on and its license as stated by its
source. ``Unspecified'' means neither the HF dataset card nor the authors' repository states a
license; we use these datasets for non-commercial research evaluation only, as their sources
describe. Training corpora are listed for provenance only: we train nothing.}
\label{tab:repro-datasets}
\begin{tabular}{@{}>{\raggedright}p{2.9cm}>{\raggedright}p{3.9cm}>{\raggedright}p{3.0cm}>{\raggedright\arraybackslash}p{3.0cm}@{}}
\toprule
\textbf{Dataset} & \textbf{Source} & \textbf{Split; role} & \textbf{License} \\
\midrule
\multicolumn{4}{@{}l}{\textit{Training corpora and in-domain evaluation}} \\
OpenWebText \citep{Gokaslan2019OpenWeb} & HF \nolinkurl{Skylion007/openwebtext} & last 100k docs of \texttt{train}; 110M training, \cref{sec:heldout,sec:judge} & CC0-1.0 \\
LM1B~\citep{chelba2013one} & HF \nolinkurl{billion-word-benchmark/lm1b} & \texttt{test}; LM1B training, \cref{sec:heldout} & Apache-2.0 (scripts); text from WMT News Crawl \\
SlimPajama-627B~\citep{cerebras2023slimpajama} & HF \nolinkurl{cerebras/SlimPajama-627B} & 800 \texttt{validation} seqs; SMDM training, \cref{sec:scaling} & Apache-2.0; component licenses apply \\
\midrule
\multicolumn{4}{@{}l}{\textit{Zero-shot and 8B perplexity} (\cref{sec:heldout})} \\
Penn Treebank~\citep{marcus1993building} & HF \nolinkurl{ptb-text-only/ptb_text_only} & \texttt{validation}; 110M & LDC User Agreement (LDC99T42) \\
Wikitext~\citep{merity2016pointer} & HF \nolinkurl{Salesforce/wikitext} (103), \nolinkurl{EleutherAI/wikitext_document_level} (2) & 103 \texttt{validation} (110M); 2 \texttt{test} (8B) & CC BY-SA 3.0 / GFDL \\
Lambada~\citep{paperno2016lambada} & OpenAI-preprocessed \nolinkurl{lambada_test.jsonl}; HF \nolinkurl{cimec/lambada} & \texttt{test}; 110M; 8B & CC BY 4.0 (\texttt{cimec}); MIT (OpenAI copy, per \nolinkurl{EleutherAI/lambada_openai}) \\
AG News~\citep{zhang2015character} & HF \nolinkurl{fancyzhx/ag_news}; \nolinkurl{sh0416/ag_news} & \texttt{test}; 110M; 8B & non-commercial research (source terms) \\
Scientific Papers~\citep{cohan2018discourse} & HF \nolinkurl{armanc/scientific_papers} & PubMed, ArXiv \texttt{validation}; 110M & Apache-2.0 (authors' repo); per-article licenses \\
\midrule
\multicolumn{4}{@{}l}{\textit{Multiple-choice benchmarks, LLaDA-8B} (\cref{sec:mc})} \\
PIQA~\citep{bisk2020piqa} & HF \nolinkurl{baber/piqa} & \texttt{validation} & AFL-3.0 (authors' repo) \\
HellaSwag~\citep{zellers2019hellaswag} & HF \nolinkurl{Rowan/hellaswag} & \texttt{validation} & MIT (authors' repo) \\
MuTual~\citep{cui2020mutual} & HF \nolinkurl{EleutherAI/mutual} & \texttt{validation} & none stated by authors \\
LSAT-LR, LSAT-RC~\citep{zhong2023agieval} & HF \nolinkurl{hails/agieval-lsat-{lr,rc}} & \texttt{test} & MIT (AGIEval, AR-LSAT repos); items \textcopyright{} LSAC \\
LogiQA (en)~\citep{liu2020logiqa,zhong2023agieval} & HF \nolinkurl{hails/agieval-logiqa-en} & \texttt{test} & MIT (AGIEval); none stated for LogiQA \\
\bottomrule
\end{tabular}
\end{table}

\paragraph{Assets and licenses.}
\Cref{tab:repro-assets,tab:repro-datasets} list every evaluated model and dataset, the exact
source it was loaded from, and its license as stated by that source. Where a source states
no license we say so instead of inferring one. All assets are used only for non-commercial research
evaluation; we redistribute none of them, and release only our evaluation code.

\paragraph{Evaluation tools.}
\begin{itemize}[leftmargin=*,itemsep=1pt,topsep=2pt]
\item \textbf{GPT-2 Large}~\citep{radford2019language} (HF \nolinkurl{openai-community/gpt2-large};
  MIT): generative perplexity, entropy, and MAUVE features.
\item \textbf{MAUVE}~\citep{pillutla2021mauve}: its reference implementation
  (\nolinkurl{github.com/krishnap25/mauve}; GPL-3.0).
\item \textbf{LM Evaluation Harness}~\citep{gao2024evalharness} (MIT): the 8B perplexity and
  multiple-choice tasks.
\item \textbf{LLM judge}: Claude Sonnet~5 (\cref{sec:judge}), queried through the Anthropic API
  under Anthropic's terms of service.
\end{itemize}

\paragraph{Code.}
\ifpreprint
Code, configurations, and the Slurm scripts behind the timings of \cref{app:cost} are available at
\url{https://github.com/kuleshov-group/duel}.
\fi

\end{document}